\documentclass[a4paper]{article}

\usepackage{microtype}
\usepackage{graphicx}
\usepackage{subfigure}
\usepackage{booktabs} 

\usepackage{natbib}
\usepackage{amsfonts}       
\usepackage{amsmath}
\usepackage{caption}
\usepackage{graphicx}
\usepackage{amssymb}
\usepackage{courier}

\usepackage{amsthm}
\newtheorem{prop}{Proposition}

\newtheorem*{theorem*}{Theorem}

\newtheorem{lemma}{\hspace{0.0cm}Lemma}

\DeclareMathOperator*{\argmin}{argmin} 

\usepackage{algorithm}
\usepackage{algcompatible}

\usepackage{color}

\usepackage[english]{babel}
\usepackage[utf8x]{inputenc}
\usepackage[T1]{fontenc}

\usepackage[a4paper,top=3cm,bottom=2cm,left=3cm,right=3cm,marginparwidth=1.75cm]{geometry}

\usepackage{amsmath}
\usepackage{graphicx}
\usepackage[colorinlistoftodos]{todonotes}
\usepackage[colorlinks=true, allcolors=blue]{hyperref}

\title{Unbiased scalable softmax optimization}
\author{Francois Fagan, Garud Iyengar\\
 {\small Department of Industrial Engineering and Operations Research}\\
{\small Columbia University}}

\begin{document}
\date{}
\maketitle


\begin{abstract}
Recent neural network and language models rely on softmax distributions with an extremely large number of categories. Since calculating the softmax normalizing constant in this context is prohibitively expensive, there is a growing literature of efficiently computable but \emph{biased} estimates of the softmax. In this paper we propose the first \emph{unbiased} algorithms for maximizing the softmax likelihood whose work per iteration is independent of the number of classes and datapoints (and no extra work is required at the end of each epoch). We show that our proposed unbiased methods comprehensively outperform the state-of-the-art on seven real world datasets.
\end{abstract}


\section{Introduction}
Under the softmax  model\footnote{Also known as the multinomial logit model.} the probability that a random variable $y$  takes on the label $\ell \in \{1,..., K\}$, is given~by  
\begin{equation}\label{eq:softmax} 
p(y=\ell|x;W) = \frac{e^{x^\top w_\ell}}{\sum_{k=1}^Ke^{x^\top w_k}},
\end{equation}
where $x\in\mathbb{R}^D$ is the covariate, $w_k\in\mathbb{R}^D$ is the
vector of parameters for the $k$-th class, and $W=[w_1, w_2,
...,w_K]\in\mathbb{R}^{D\times K}$ is the parameter matrix. Given a
dataset of $N$ label-covariate pairs $\mathcal{D} = \{(y_i,
x_i)\}_{i=1}^N$, the ridge-regularized maximum log-likelihood problem~is
given by 
\begin{align}\label{eq:original_log_likelihood}
L(W) &= \sum_{i=1}^N x_i^\top w_{y_i} -\log(\sum_{k=1}^Ke^{x_i^\top w_k}) -\frac{\mu}{2}\|W\|_2^2,
\end{align}
where $\|W\|_2$ denotes the Frobenius norm.


The softmax is a fundamental and ubiquitous distribution, with 
applications in fields such as economics and
biomedicine~\citep{rust1993customer,kirkwood2010essential,gopal2013distributed}
and appears as a convex surrogate for the (hard) maximum loss in discrete
optimization~\citep{maddison2016concrete} and network
flows~\citep{shahrokhi1990maximum}. 
This paper focusses on how to maximize~(\ref{eq:original_log_likelihood})
when $N$, $K$, $D$ are all large.   
Large values for $N$,$K$,$D$ are increasingly common in modern applications such as
natural language processing and recommendation systems, where $N$, $K$, and
$D$ can each be on the order of millions or billions
\citep{chelba2013one,partalas2015lshtc}.


A natural approach to maximizing $L(W)$ with large values for $N$, $K$ 
and $D$ is to use Stochastic Gradient Descent (SGD), sampling a mini-batch
of datapoints each iteration. However when $K$ and $D$ are large, the $O(KD)$
cost of calculating the normalizing sum $\sum_{k=1}^Ke^{x_i^\top w_k}$ in
the stochastic gradients can be prohibitively expensive.  
Several approximations that avoid calculating the normalizing sum have been
proposed to address this difficulty. These include tree-structured methods
\citep{bengio2003quick,
daume2016logarithmic,grave2016efficient,jernite2016simultaneous},
sampling methods
\citep{bengio2008adaptive,mnih2012fast,ji2015blackout,
joshi2017aggressive}
and self-normalization \citep{andreas2015and}. Alternative models 
such as the  spherical family of losses
\citep{de2015exploration,vincent2015efficient} that do not require
normalization have been proposed to sidestep the issue entirely
\citep{martins2016softmax}. \citet{krishnapuram2005sparse} avoid
calculating the sum using a maximization-majorization approach based on
lower-bounding the eigenvalues of the Hessian matrix.  
All\footnote{The method of \citet{krishnapuram2005sparse} \emph{does} converge to
  the optimal MLE, but has $O(ND)$ runtime per iteration which is not
  feasible for large~$N$ and $D$.} of these approximations are computationally
tractable for large $N$, $K$ and $D$, but are unsatisfactory in that they are
\emph{biased} and do not converge to the optimal $W^\ast = \text{argmax}
\,L(W)$. 


Recently\footnote{This same idea has appeared multiple times in the literature. For example \citep{ruiz2018augment} use a similar idea for variational inference of the softmax.} \citet{raman2016ds} showed how to recast
(\ref{eq:original_log_likelihood}) as a double-sum over $N$ and $K$. This
formulation is amenable to SGD that samples only one datapoint and
class in each iteration, reducing the per iteration cost to $O(D)$. However, vanilla SGD applied to this formulation is unstable in that the stochastic gradients may have high variance and a high dynamic range leading to computational overflow errors. \citet{raman2016ds} deal with this instability by occasionally
calculating the normalizing sum for all datapoints at a cost of
$O(NKD)$. Although this achieves stability, its high cost nullifies the
benefit of the cheap $O(D)$ per iteration cost. 



In this paper we propose two robust \emph{unbiased} SGD algorithms for optimizing double-sum formulations of the softmax likelihood. 
The first is an implementation of Implicit SGD, a stochastic
gradient method that is known to be more stable than vanilla SGD, and yet
has similar convergence properties~\citep{toulis2016towards}. We show
that the Implicit SGD updates for the double-sum formulation can be
efficiently computed using a bisection method with tight initial
bounds. Furthermore, we guarantee the stability of Implicit SGD by proving
that the step size is asymptotically linearly bounded (unlike vanilla SGD
which is exponentially bounded). 
The second algorithm is a new SGD method called U-max, that is
guaranteed to have bounded gradients and converges to the optimal solution
of~(\ref{eq:original_log_likelihood}) for all sufficiently small learning
rates. This method is particularly suited to situations where calculating simultaneous inner products is cheap (for example when using GPUs).

We compare the performance of U-max and Implicit SGD to the (biased)
state-of-the-art methods for maximizing the softmax likelihood which cost
$O(D)$ per iteration. Both U-max and Implicit SGD outperform all other
methods.  
Implicit SGD has the best performance with an average log-loss 4.44 times
lower than the previous state-of-the-art biased methods.  


In summary, our contributions in this paper are that we:
\begin{enumerate}
\item 
  Develop an alternative
  softmax double-sum formulation with 
  gradients of smaller magnitude as compared to that in \citet{raman2016ds}
  (Section~\ref{sec:double_sum}).  

\item Derive an efficient implementation of Implicit SGD using a bisection
  method, analyze its runtime and bound its step size
  (Section~\ref{sec:I_SGD}).  

\item Propose the U-max algorithm to stabilize the vanilla SGD updates and
  prove its convergence (Section~\ref{sec:u_max}).  

\item Conduct experiments showing that both U-max and Implicit SGD
  outperform the previous state-of-the-art, with Implicit SGD having the
  best performance (Section~\ref{sec:experiments}). 

\end{enumerate}

\section{Convex double-sum formulation}\label{sec:double_sum}
\subsection{Derivation of double-sum}\label{sec:derivation_double_sum}
In order to have an SGD method that samples both datapoints and classes
each iteration, we need to represent~(\ref{eq:original_log_likelihood}) as
a double-sum over datapoints and classes.  
We begin by rewriting~(\ref{eq:original_log_likelihood}) in a more convenient form,
\begin{align}
L(W) &= \sum_{i=1}^N  -\log(1 + \sum_{k\neq y_i}e^{x_i^\top (w_k -
       w_{y_i})}) -\frac{\mu}{2}\|W\|_2^2. \label{eq:OVE_log_likelihood} 
\end{align}
The key to converting~(\ref{eq:OVE_log_likelihood}) into its double-sum
representation is to express the negative logarithm using its
convex conjugate\footnote{This trick is related to the bounds given in \citep{gopal2013distributed}.}: 
\begin{align}
- \log (a) &=  \max_{v< 0}\{av - (-\log(-v)-1)\}\nonumber\\
&=  \max_{u}\{-u-\exp(-u)a+1\} \label{eq:trick}
\end{align}
where $u = -\log(-v)$ and the optimal value of $u$ is $u^\ast(a) =
\log(a)$. Applying~(\ref{eq:trick}) to each of the logarithmic terms in
(\ref{eq:OVE_log_likelihood}) yields $L(W) = -\min_{u\geq 0}\left\{f(u,W) \right\} + N$ where
\begin{equation}\label{eq:f}
f(u,\! W)\!=\!\sum_{i=1}^N\!\sum_{k\neq y_i}\!\frac{u_i +
  e^{-u_i}}{K-1}+e^{x_i^\top (w_k - w_{y_i})-u_i}\!+ \frac{\mu}{2}\|W\|_2^2 
\end{equation}
is our double-sum representation that we seek to minimize. 
Clearly $f$ is a jointly convex function in $u$ and~$W$. The variable $u_i$ can be thought of as an approximation to the log-normalizer, as its optimal solution is
$u_i^\ast(W) = \log(1+\sum_{k\neq y_i}e^{x_i^\top (w_k - w_{y_i})})\geq0$. In
Appendix~\ref{app:f_properties} we prove that the optimal $u$ and
$W$ are contained in a compact convex set and that $f$ is strongly convex
within this set. Thus performing projected-SGD on $f$ is
guaranteed to converge to a unique optimum with a convergence rate of
$O(1/T)$ where $T$ is the number of iterations
\citep{lacoste2012simpler}. 

\subsection{Instability of vanilla SGD}\label{sec:numerical_instability}
The challenge in optimizing $f$ using SGD is that the gradients can have 
very large magnitudes.    
Observe that ${f = \mathbb{E}_{ik}[f_{ik}]}$ where $i\sim \text{unif}(\{1,...,N\})$, $k\sim
\text{unif}(\{1,...,K\} - \{y_i\})$ and  
\begin{equation}
f_{ik}(u,W) = N\left( u_i + e^{-u_i}+(K-1)e^{x_i^\top (w_k -
  w_{y_i})-u_i}\right) + \frac{\mu}{2}(\beta_{y_i} \|w_{y_i}\|_2^2+
  \beta_k \|w_k\|_2^2), \label{eq:stochastic_f}
\end{equation}
where $\beta_j = \frac{N}{n_j+(N-n_j)/(K-1)}$ is the inverse of the
probability of class $j$ being sampled either through $i$ or $k$, and $n_j
= |\{i:y_i=j\}|$. 
The corresponding stochastic gradient is:
\begin{align}
\nabla_{w_k} f_{ik} &= N(K-1)e^{x_i^\top (w_k - w_{y_i})-u_i}x_i +
                           \mu \beta_k w_k \nonumber\\ 
\nabla_{w_{y_i}} f_{ik}&= -N(K-1)e^{x_i^\top (w_k - w_{y_i})-u_i}x_i
                              + \mu \beta_{y_i} w_{y_i} \nonumber\\ 
\nabla_{w_j} f_{ik}&= 0 \qquad \forall j\notin\{k,y_i\}\nonumber\\ 
\nabla_{u_i} f_{ik}&=   - N(K-1)e^{x_i^\top (w_k - w_{y_i})-u_i} +
                          N(1-e^{-u_i}) \label{eq:vanilla_stoch_updates} 
\end{align}
If $u_i$ is at its optimal value $u_i^\ast(W) = \log(1+\sum_{k\neq
  y_i}e^{x_i^\top (w_k - w_{y_i})})$ then $e^{x_i^\top (w_k -
  w_{y_i})-u_i} \leq 1$ and the magnitude of the $N(K-1)e^{x_i^\top (w_k - w_{y_i})-u_i}$ terms in the gradient are bounded by $N(K-1)\|x_i\|_2$. However if $u_i\ll
x_i^\top (w_k - w_{y_i})$, then $e^{x_i^\top (w_k - w_{y_i})-u_i}\ggg 1$
and the magnitude of the gradients can become extremely large.

Extremely large gradients lead to two major problems: (a) 
they could lead to overflow errors and 
cause the
algorithm to crash, (b) they result in the stochastic gradient having high
variance, which leads to slow convergence\footnote{The convergence rate of
  SGD is inversely proportional to the second moment of its gradients
  \citep{lacoste2012simpler}.}. In Section~\ref{sec:experiments} we show
that  these problems occur in practice and make vanilla SGD both an
unreliable and inefficient method\footnote{The same problems arise if we
  approach optimizing 
(\ref{eq:OVE_log_likelihood}) via stochastic composition optimization
\citep{wang2016accelerating}. As is shown in Appendix~\ref{app:Mendi},
stochastic composition optimization yields near-identical expressions for
the stochastic gradients in~(\ref{eq:vanilla_stoch_updates}) and has the
same stability issues.}.

The sampled softmax optimizers in the literature
\citep{bengio2008adaptive, mnih2012fast, ji2015blackout,
  joshi2017aggressive} do not have the issue of large magnitude
gradients. Their gradients are bounded by $N(K-1)\|x_i\|_2$ 
since 
their approximations ensure that 
$u_i^\ast(W)  > x_i^\top (w_k -
w_{y_i})$. For example, in one-vs-each \citep{titsias2016one},
$u_i^\ast(W)$ is approximated by $\log(1 + e^{x_i^\top (w_k - w_{y_i})})
> x_i^\top (w_k - w_{y_i})$. However, since these methods only approximate
$u_i^\ast(W)$, the iterates do converge to the optimal~$W^\ast$. 

The goal of this paper is to design reliable and efficient SGD algorithms
for optimizing the {double-sum} formulation in~(\ref{eq:f}). We
propose two such methods: Implicit SGD (Section~\ref{sec:I_SGD}) and U-max (Section~\ref{sec:u_max}). But before we
introduce these methods we should establish that (\ref{eq:f}) is a good choice for
the double-sum formulation.

\subsection{Choice of double-sum formulation}\label{sec:double_sum_formulation}
The double-sum in~(\ref{eq:f}) is different to that of \citet{raman2016ds}. Their formulation can be derived by applying the convex conjugate substitution to
(\ref{eq:original_log_likelihood}) instead of
(\ref{eq:OVE_log_likelihood}). The resulting equations are  
$$L(W) =
-\min_{\bar{u}}\left\{\frac{1}{N}\sum_{i=1}^N\frac{1}{K-1}\sum_{k\neq y_i}
  \bar{f}_{ik}(\bar{u},W)\right\}+N$$ 
where
\begin{equation}
\bar{f}_{ik}(\bar{u},W) =  N\big( \bar{u}_i -x_i^\top w_{y_i} +
  e^{x_i^\top w_{y_i}-\bar{u}_i} +(K-1)e^{x_i^\top
    w_k-\bar{u}_i}\big)+ \frac{\mu}{2}(\beta_{y_i} \|w_{y_i}\|_2^2+
\beta_k \|w_k\|_2^2) \label{eq:f_original_log_likelihood}
\end{equation}
and the optimal solution for $\bar{u}_i$ is $\bar{u}_i^\ast(W^\ast) =
\log(\sum_{k=1}^K e^{x_i^\top w_k^\ast})$. The only difference between the
formulations is the reparameterization $\bar{u}_i = u_i + x_i^\top w_{y_i}$. 

Although either double-sum formulations can be used as a basis for SGD, our
formulation in~\eqref{eq:f} tends to have smaller magnitude stochastic gradients, and hence
faster convergence.  
To see this on a high level, note that typically ${x_i^\top w_{y_i} = \mbox{argmax}_k\{ x_i^\top
w_k \}}$ and so the $\bar{u}_i$, $x_i^\top w_{y_i}$ and
$e^{x_i^\top w_{y_i}-\bar{u}_i}$ terms are of the greatest
magnitude  in
(\ref{eq:f_original_log_likelihood}). Although at optimality these terms should roughly
cancel, this will not be the case during the early stages of optimization, leading to
stochastic gradients of large magnitude. 
In contrast, the function $f_{ik}$ in~(\ref{eq:stochastic_f}) only has
$x_i^\top w_{y_i}$ appearing as a negative exponent, and so if $x_i^\top
w_{y_i}$ is large then the magnitude of the stochastic gradients will be
small. A more rigorous version of this argument is presented in Appendix~\ref{app:double_sum} and in Section~\ref{sec:experiments} we present numerical results confirming that our double-sum formulation leads to faster convergence.

\section{Stable SGD methods}\label{sec:stable_SGD}

\subsection{Implicit SGD}\label{sec:I_SGD}
One method that solves the large gradient problem is Implicit
SGD\footnote{Also known to as an ``incremental proximal algorithm''
  \citep{bertsekas2011incremental} or ``stochastic proximal iteration'' \citep{ryu2014stochastic}.} \citep{bertsekas2011incremental, ryu2014stochastic, toulis2015implicit, 
  toulis2016towards}. Implicit SGD uses the update
equation 
\begin{align}\label{eq:Implicit_SGD_formula}
\theta^{(t+1)} = \theta^{(t)} - \eta_t\nabla f(\theta^{(t+1)}, \xi_t),
\end{align}
where $\theta^{(t)}$ is the value of the $t^{th}$ iterate, $f$ is the
function we seek to minimize and $\xi_t$ is a random variable controlling
the stochastic gradient such that $\nabla f(\theta) =
\mathbb{E}_{\xi_t}[\nabla f(\theta, \xi_t)]$. The update
(\ref{eq:Implicit_SGD_formula}) differs from 
vanilla SGD in that $\theta^{(t+1)}$ appears on both the left and right
side of the equation, whereas in vanilla SGD it appears only on the left
side.  
In our case $\theta = (u,W)$
and $\xi_t = (i_t,k_t)$ with $\nabla f(\theta^{(t+1)}, \xi_t) = \nabla
f_{i_t,k_t}(u^{(t+1)},W^{(t+1)})$. 

Although Implicit SGD has similar convergence rates to vanilla SGD, it has
other properties that can make it preferable over vanilla SGD. It is more
robust to the learning rate \citep{toulis2016towards}, which important
since a good value for the learning rate is never known a priori, and is
provably more stable \citep[Section
5]{ryu2014stochastic}. Another property, which is of particular interest
to our problem, is that it has smaller step sizes. 
\begin{prop}\label{prop:implicit_general_step_size}
Consider applying Implicit SGD to optimizing $ f(\theta) =
\mathbb{E}_{\xi}[f(\theta, \xi)]$ where $f(\theta, \xi)$ is $m$-strongly convex for all~$\xi$. Then 
\begin{align*}
\|\nabla f(\theta^{(t+1)}, \xi_t)\|_2\!\leq\! \|\nabla f(\theta^{(t)}, \xi_t)\|_2 - m\|\theta^{(t+1)} - \theta^{(t)}\|_2
\end{align*}
and so the Implicit SGD step size is smaller than that of vanilla SGD.
\end{prop} 
\begin{proof}
The proof is provided in Appendix~\ref{app:general_implicit_step_bound}.
\end{proof}

The bound in Proposition~\ref{prop:implicit_general_step_size} can be
tightened for our particular problem. Unlike vanilla SGD whose step size
magnitude is \emph{exponential} in $x_i^\top (w_k - w_{y_i})-u_i$, as
shown in~(\ref{eq:vanilla_stoch_updates}), for Implicit SGD the step size
is asymptotically \emph{linear} in $x_i^\top (w_k - w_{y_i})-u_i$. This
effectively guarantees that Implicit SGD cannot suffer from computational
overflow. 

\begin{prop}\label{prop:step_size}
Consider the Implicit SGD algorithm where in each iteration only one datapoint $i$ and one class $k\neq
y_i$ is sampled. The magnitude of its step size in $W$ is ${O(x_i^\top(\frac{w_k}{1+\eta\mu\beta_k} - \frac{w_{y_i}}{1+\eta\mu\beta_{y_i}})-u_i)}$.
\end{prop} 
\begin{proof}
The proof is provided in Appendix~\ref{app:Implicit_SGD_step_size_bound}.
\end{proof}

The major difficulty in applying Implicit SGD is that in each iteration
one has to compute a 
solution to~(\ref{eq:Implicit_SGD_formula}) \citep[Section 6]{ryu2014stochastic}. 
The tractability of this procedure is problem dependent. We show that
computing a solution to \eqref{eq:Implicit_SGD_formula} is indeed
tractable for the problem considered in this paper. 
The details are laid out in full in 
Appendix~\ref{app:Implicit_SGD}. 

\begin{prop}\label{prop:multiple}
Consider the Implicit SGD algorithm where in each iteration $n$ datapoints
and $m$ classes are sampled. The Implicit SGD update $\theta^{(t+1)}$
can be computed to within $\epsilon$ accuracy in runtime $O(n^2(n+m)\log(\epsilon^{-1})+nmD)$. 
\end{prop} 
\begin{proof}
The proof is provided in Appendix~\ref{app:Implicit_SGD_multiple}.
\end{proof}

In Proposition~\ref{prop:multiple} the $\log(\epsilon^{-1})$ factor comes
from applying a first order method to solve the strongly convex Implicit
SGD update equation. It may be the case that performing this optimization
is more expensive than the $O(nmD)$ cost of computing the $x_i^\top w_k$ inner products, and so
each iteration of Implicit SGD may be significantly slower than that of
vanilla SGD.  

Fortunately, in certain cases we can improve the runtime of solving the implicit update.
If $n=1$ and we just sample one datapoint per iteration then it is possible to reduce the update to solving just a univariate strongly convex optimization problem (see Appendix \ref{app:Implicit_SGD_single_multiple} for details). Furthermore, when $m=1$ and only one class is sampled per iteration then we can derive upper and lower bounds on the one-dimensional variate to be optimized over. The optimization problem can then be solved using a bisection method, with an explicit upper bound on its cost. 

\begin{prop}\label{prop:binary}
Consider the Implicit SGD algorithm with learning rate $\eta$ where in each iteration only one datapoint $i$ and one class $k\neq
y_i$ is sampled. The Implicit SGD iterate $\theta^{(t+1)}$ can be
computed to within $\epsilon$ 
accuracy with only two $D$-dimensional vector inner products and at most $\log_2(\epsilon^{-1}) + \log_2(|x_i^\top(\frac{w_k}{1+\eta\mu\beta_k} - \frac{w_{y_i}}{1+\eta\mu\beta_{y_i}})-u_i|+2\eta N \|x_i\|_2^2+\log(2K))$ bisection method function evaluations. 
\end{prop} 
\begin{proof}
The proof is provided in Appendix~\ref{app:Implicit_SGD_simpled} and the pseudocode is presented in Algorithm~\ref{alg:isgd} in the appendix.
\end{proof}
For any reasonably large dimension $D$, the cost of the two
$D$-dimensional vector inner-products will outweigh the cost of the
bisection, and Implicit SGD with $n=m=1$ will have roughly the same speed per iteration as vanilla SGD with $n=m=1$. This is empirically confirmed for seven real-world datasets in Section \ref{sec:experiment_setup}. 

However, if calculating inner products is relatively cheap (for example if $D$ is small or GPUs are used), then Implicit SGD will be slower than vanilla SGD. The U-max algorithm, presented next, is stable in the same way Implicit SGD is but has the same runtime as vanilla SGD. This makes U-max an ideal choice when inner products are cheap.

\subsection{U-max method}\label{sec:u_max}

As explained in Section~\ref{sec:numerical_instability}, vanilla SGD
has large gradients when $u_i\ll x_i^\top (w_k - w_{y_i})$. 
This can only occur when $u_i$
is less than its optimum value for the current $W$,
since $u_i^\ast(W) = \log(1 + \sum_{j\neq y_i} e^{x_i^\top (w_k -
  w_{y_i})}) \geq x_i^\top (w_k - w_{y_i})$. A simple remedy is to set
$u_i= \log(1+e^{x_i^\top (w_k - w_{y_i})})$ whenever $u_i\ll x_i^\top (w_k
- w_{y_i})$. Since $\log(1+e^{x_i^\top (w_k - w_{y_i})}) > x_i^\top (w_k -
w_{y_i})$ this guarantees that $u_i> x_i^\top (w_k - w_{y_i})$ and so the
gradients will be bounded. It also brings $u_i$ closer\footnote{Since $u_i <
  x_i^\top (w_k - w_{y_i}) < \log(1+e^{x_i^\top (w_k - w_{y_i})})<\log(1 +
  \sum_{j\neq y_i} e^{x_i^\top (w_k - w_{y_i})})=u_i^\ast(W)$.} to its
optimal value for the current $W$ and thereby decreases the the objective
$f(u, W)$. 

This is exactly the mechanism behind the U-max algorithm --- see
Algorithm~\ref{alg:umax} in  Appendix~\ref{app:umax_alg} for its
pseudocode. U-max is the same as  vanilla SGD except for two
modifications: (a) $u_i$ is set equal to $\log(1+e^{x_i^\top (w_k -
  w_{y_i})})$  whenever $u_i \leq \log(1+e^{x_i^\top (w_k - w_{y_i})}) -
\delta $ for some threshold $\delta>0$, (b)   $u_i$ is projected onto
$[0,B_u]$, and $W$ onto $\{W:\|W\|_2\leq B_W\}$, where $B_u$ and $B_W$ are
set so that the optimal $u^\ast_i \in [0,B_u]$ and the optimal $W^\ast$
satisfies $\|W^\ast\|_2 \leq B_W$. See Appendix~\ref{app:f_properties} for
more details on how to set $B_u$ and $B_W$.  

\begin{prop}\label{thm:Umax} 
Suppose $B_f \geq  \max_{ik} \|\nabla
f_{ik} (u,W)\|_2$ for all  $\|W\|_2^2\leq B_W^2$ and  $0\leq u \leq
B_u$. Suppose the learning rate $\eta_t \leq \delta^2/(4 B_f^2)$, then
U-max with threshold $\delta$ converges to the optimum of
(\ref{eq:original_log_likelihood}), and the rate of convergence is at least as fast as
SGD with the same learning rate. 
\end{prop}
\begin{proof}
The proof is provided in Appendix~\ref{app:proof_Umax}.
\end{proof}
U-max directly resolves the problem of extremely large
gradients. Modification (a) ensures that $\delta\geq x_i^\top (w_k -
w_{y_i}) -u_i$ (otherwise $u_i$ would be increased to $\log(1+e^{x_i^\top
  (w_k - w_{y_i})})$) and so the magnitude of the U-max gradients are
bounded above by  $N(K-1)e^\delta\|x_i\|_2$.  

In U-max there is a trade-off between the gradient magnitude and learning
rate that is controlled by $\delta$. For Proposition~\ref{thm:Umax} to apply
we require that the learning rate $\eta_t \leq \delta^2/(4 B_f^2)$. A
small $\delta$ yields small magnitude gradients, which makes convergence
fast, but necessitates a small $\eta_t$, which makes convergence slow.

As presented above U-max only samples one datapoint and class per iteration, but it trivially generalizes to multiple datapoints and classes. If $n$ datapoints and $m$ classes are sampled, then the runtime is $O(nmD)$, due to the vector inner-product calculations. This is the same runtime as vanilla SGD as well as the state-of-the-art biased methods such as Noise Contrastive Estimation 
\citep{mnih2012fast}, Importance Sampling \citep{bengio2008adaptive}
and One-Vs-Each \citep{titsias2016one}.

\section{Experiments}\label{sec:experiments}
Two sets of experiments were conducted to assess the performance of the proposed methods. The first compares U-max and Implicit SGD to the state-of-the-art over seven real world datasets. The second investigates the difference in performance between the two double-sum formulations discussed in Section~\ref{sec:double_sum_formulation}. We begin by specifying the experimental setup and then move onto the results. 

\subsection{Experimental setup}\label{sec:experiment_setup}
 
\textbf{Data.} We used the MNIST, 
Bibtex, 
Delicious, 
Eurlex, 
AmazonCat-13K, 
Wiki10, 
and WikiSmall 
datasets\footnote{All of the datasets were downloaded from
   \url{http://manikvarma.org/downloads/XC/XMLRepository.html}, except
   WikiSmall which was obtained from
   \url{http://lshtc.iit.demokritos.gr/}.}, the properties of which are
 summarized in Table~\ref{tbl:datasets}. Most of the datasets are
 multi-label and, as is standard practice \citep{titsias2016one}, we took
 the first label as being the true label and discarded the remaining
 labels. To make the computation more manageable, we truncated the number
 of features to be at most 10,000 and the training and test size to be at
 most 100,000. If, as a result of the dimension truncation, a datapoint
 had no non-zero features then it was discarded. The features of each
 dataset were normalized to have unit $L_2$ norm. All of the datasets were
 pre-separated into training and test sets. We only focus on the
 performance on the algorithms on the training set, as the goal in this
 paper is to investigate how best to optimize the softmax likelihood,
 which is given over the training set. \\
  
 \begin{table}[t]
\caption{Datasets with a summary of their properties. Where the number of classes, dimension or number of examples has been altered, the original value is displayed in brackets.}
\label{tbl:datasets}
\vskip 0.15in
\begin{center}
\begin{small}
\begin{sc}
\begin{tabular}{llllr}
\toprule
Data set & Classes & Dimension & Examples \\
\midrule
MNIST & 10 &780 & 60,000 \\
Bibtex & 147 (159)&1,836&  4,880\\
Delicious & 350 (983) &500 &  12,920 \\
Eurlex & 838 (3,993) &5,000 & 15,539 \\
AmazonCat-13K & 2,709 (2,919) & 10,000 (203,882) & 100,000 (1,186,239)\\
Wiki10 &4,021 (30,938) & 10,000 (101,938) &14,146 \\
WikiSmall & 18,207 (28,955)& 10,000 (2,085,164) & 90,737 (342,664)\\
\bottomrule
\end{tabular}
\end{sc}
\end{small}
\end{center}
\vskip -0.1in
\end{table}

\noindent\textbf{Algorithms.} We compared our algorithms to the state-of-the-art
methods for optimizing the softmax which have runtime $O(D)$ per
iteration\footnote{\citet{raman2016ds} have runtime $O(NKD)$ per epoch, which is equivalent to $O(KD)$ per iteration. This is a factor of $K$ slower than the methods we compare against. In most of our experiments, the second epoch of Raman would not have even started by the time our algorithms have already nearly converged.}. The
competitors include Noise Contrastive Estimation (NCE)
\citep{mnih2012fast}, Importance Sampling (IS) \citep{bengio2008adaptive}
and One-Vs-Each (OVE) \citep{titsias2016one}. Note that these methods are
all biased and will not converge to the optimal softmax MLE, but, perhaps,
something close to it. For these algorithms we set $n=100, m=5$, which are
 standard settings\footnote{We also experimented setting $n=1, m=5$
  in these methods and there was virtually no difference in performance except the
  runtime was slower. 
  }. For Implicit SGD we chose to implement the version in
Proposition~\ref{prop:binary} which has $n=1, m=1$ and used Brent's method as out bisection method solver.

For U-max and vanilla SGD we set $n=1, m=5$ and for U-max the threshold parameter $\delta =1$. For both methods we also experimented with $m=1$ but obtained significantly better performance with $m=5$. The probable reason is that having a larger $m$ value decreases the variance of the gradients, making the algorithms more stable with higher learning rates and thereby improving convergence.

The ridge regularization parameter $\mu$ was set to zero and the classes were sampled uniformly for all algorithms. \\

\noindent\textbf{Epochs, losses and runtimes.} Each algorithm was run for $50$ epochs on each
dataset. The learning rate was decreased by a factor of 0.9 each
epoch. Both the prediction error and log-loss
(\ref{eq:original_log_likelihood}) were recorded at the end of 10 evenly
spaced epochs over the 50 epochs.  

The OVE, NCE, IS, Vanilla and U-max algorithms have virtually the same
runtime per iteration and so their relative performance can be gauged by
plotting their log-loss over the epochs. Since Implicit SGD has to
solve an inner optimization problem each iteration, its runtime will be
slower than that of other algorithms with $n=1,m=1$, but may be faster
than algorithms with $n=1, m>1$. Thus plotting its performance over the epochs
may yield an inaccurate comparison to the other algorithms with respect to
runtime. 

To investigate this we measured the runtime of Implicit
SGD with $n=m=1$ vs vanilla SGD\footnote{Any of OVE, NCE, IS, Vanilla
  and U-max could have been used since their runtimes are virtually
  identical.} with $n=1, m=5$ for 50 epochs on each
dataset. To make the runtime comparison as fair as possible, both
algorithms were coded in a standard NumPy framework. 
The runtime of Implicit SGD is
$0.65\pm 0.15$ times that of vanilla SGD (see Table~\ref{tbl:runtimes} in Appendix~\ref{app:runtime_results} for runtimes.)\footnote{As noted above, vanilla SGD with $m=5$ performed significantly better than with $m=1$, thus we compare to the $m=5$ runtime. The runtime of Implicit SGD was on average $1.04\pm 0.07$ times that of vanilla SGD with $n=m=1$ for both methods.}.
 Although these results are data, 
implementation and hardware dependent, 
they strongly indicate that Implicit SGD with $n=m=1$ is faster than
vanilla SGD (or any similar method) with $n=1, m=5$. Thus plotting the log-loss
over the epochs gives a \emph{conservative estimate} of Implicit SGD's
relative performance with respect to runtime.\\

\noindent\textbf{Learning rate.} The magnitude of the gradient differs in each
algorithm, due to either under- or over-estimating the normalizing constant from~(\ref{eq:original_log_likelihood}). To set a reasonable
learning rate for each algorithm on each dataset, we ran them on 10\% of the training data with initial
learning rates\footnote{The learning rates are divided by $N$ to counter the stochastic gradient being proportional to $N$ and thereby make the step size independent of $N$.} $\eta = 10^{0,\pm1,\pm2,\pm3} / N$ . The learning rate with the best performance after 50 epochs is then used when the algorithm is applied to the full dataset.
The tuned learning rates are presented in Table~\ref{tbl:tuned_learning_rates}.
Note that vanilla SGD requires a very small learning rate, otherwise it suffered from overflow. On average the tuned vanilla SGD learning rate is 3,019 times smaller than Implicit SGD's and 319 times smaller than U-max's.\\

\begin{table}[t]
\caption{Tuned initial learning rates for each algorithm on each
  dataset. The learning rate in $10^{0,\pm1,\pm2,\pm3}/N$ with the lowest
  log-loss after 50 epochs using only 10\% of the data is
  displayed. Vanilla SGD applied to AmazonCat, Wiki10 and WikiSmall
  suffered from overflow with a learning rate of $10^{-3}/N$, but was
  stable with smaller learning rates (the largest learning rate for which
  it was stable is displayed).} 
\label{tbl:tuned_learning_rates}
\vskip 0.15in
\begin{center}
\begin{small}
\begin{sc}
\begin{tabular}{lllllll}
\toprule
Data set & OVE & NCE & IS & Vanilla & Umax & Implicit \\
\midrule
MNIST & $10^{1}$  & $10^{1}$  & $10^{1}$  & $10^{-2}$  & $10^{1}$  & $10^{-1}$\\
Bibtex & $10^{2}$  & $10^{2}$  & $10^{2}$  & $10^{-2}$  & $10^{-1}$  & $10^{1}$\\
Delicious & $10^{1}$  & $10^{3}$  & $10^{3}$  & $10^{-3}$  & $10^{-2}$  & $10^{-2}$\\
Eurlex & $10^{-1}$  & $10^{2}$  & $10^{2}$  & $10^{-3}$  & $10^{-1}$  & $10^{1}$\\
AmazonCat & $10^{1}$  & $10^{3}$  & $10^{3}$  & $10^{-5}$  & $10^{-2}$  & $10^{-3}$\\
Wiki10 & $10^{-2}$  & $10^{3}$  & $10^{2}$  & $10^{-4}$  & $10^{-2}$  & $10^{0}$\\
{WikiSmall} & $10^{3}$  & $10^{3}$  & $10^{3}$  & $10^{-4}$  & $10^{-3}$  & $10^{-3}$\\
\bottomrule
\end{tabular}
\end{sc}
\end{small}
\end{center}
\vskip -0.1in
\end{table}

\begin{figure*}[h!]
\centering
\begin{minipage}{.24\textwidth}
  \centering
  \includegraphics[width=.59\linewidth]{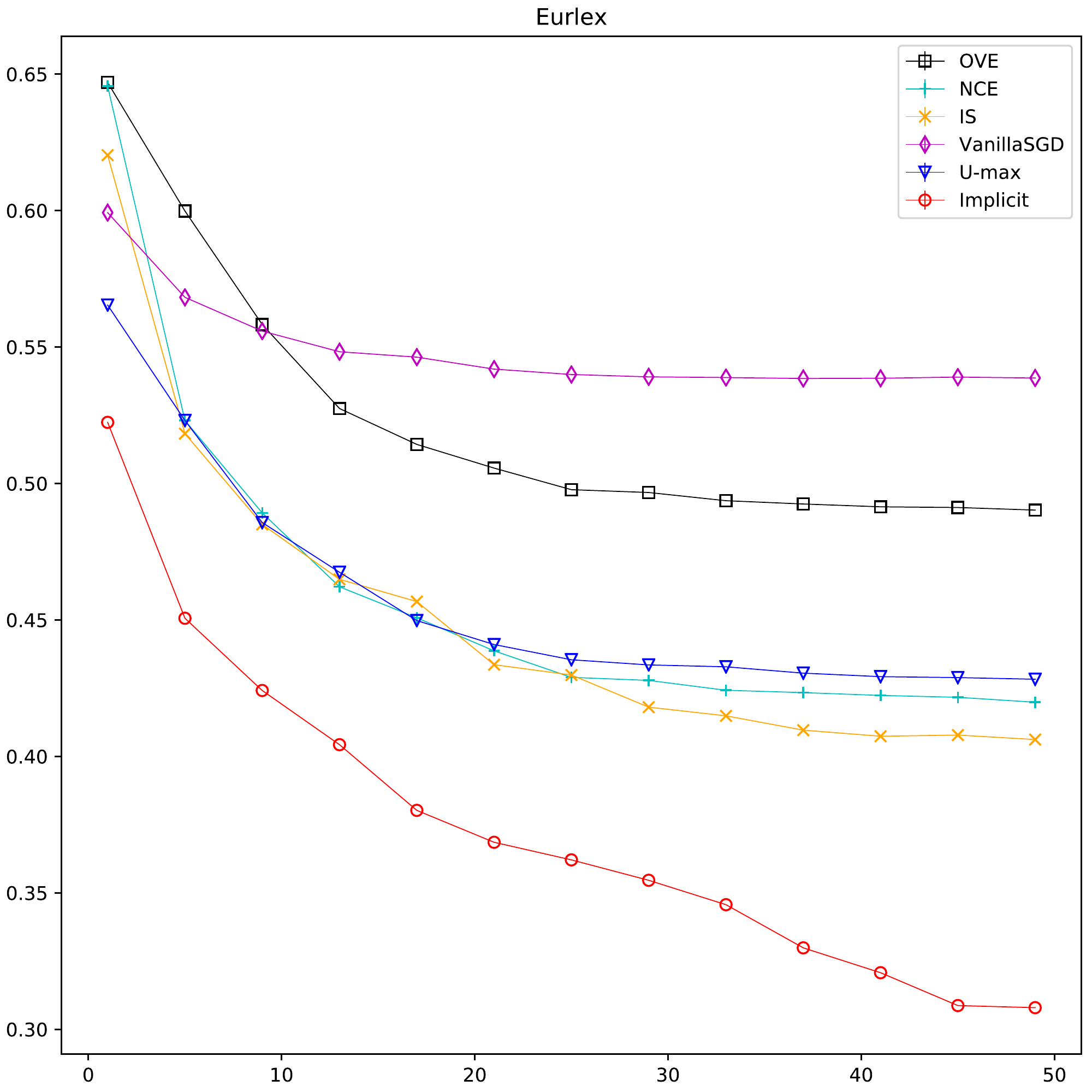}
\end{minipage}%
\hfill
\begin{minipage}{.24\textwidth}
  \centering
  \includegraphics[width=.99\linewidth]{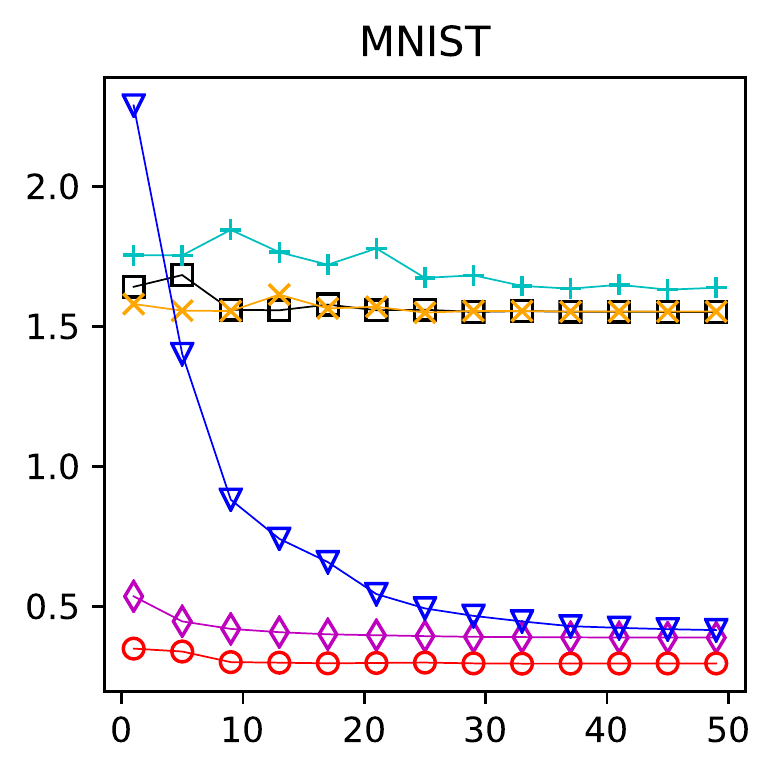}
\end{minipage}
\begin{minipage}{.24\textwidth}
  \centering
  \includegraphics[width=.99\linewidth]{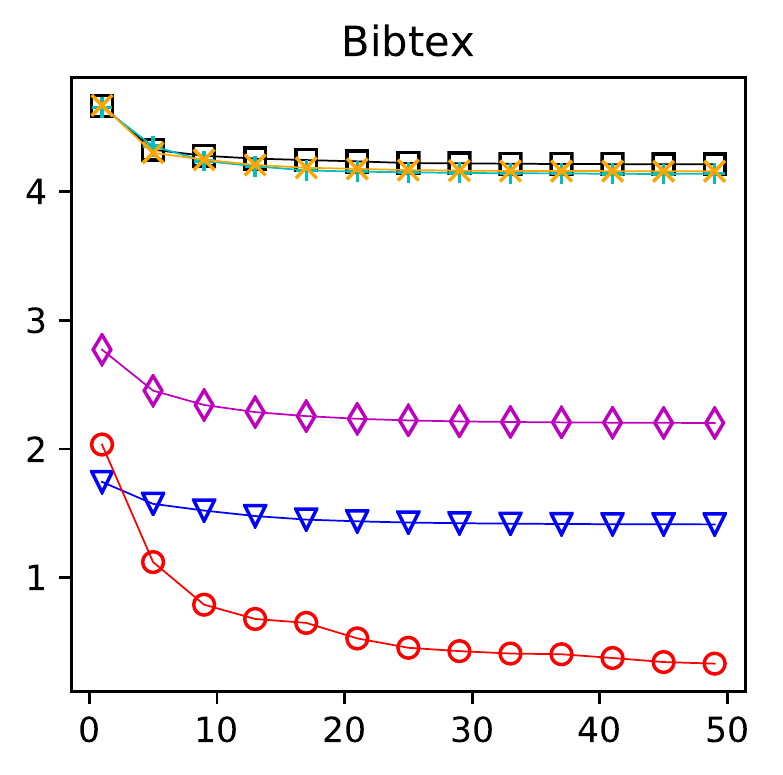}
\end{minipage}%
\hfill
\begin{minipage}{.24\textwidth}
  \centering
  \includegraphics[width=.99\linewidth]{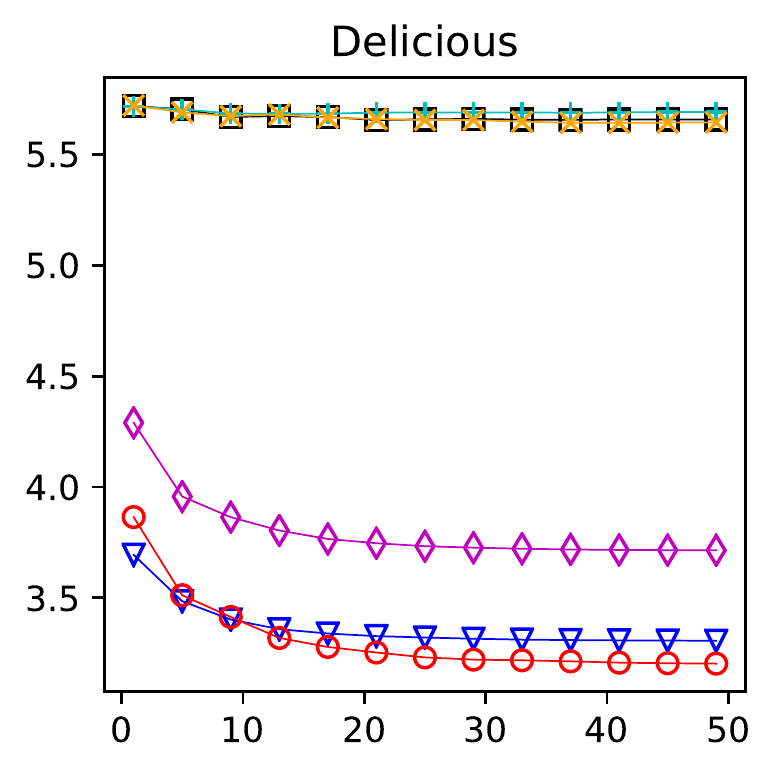}
\end{minipage}
\begin{minipage}{.24\textwidth}
  \centering
  \includegraphics[width=.99\linewidth]{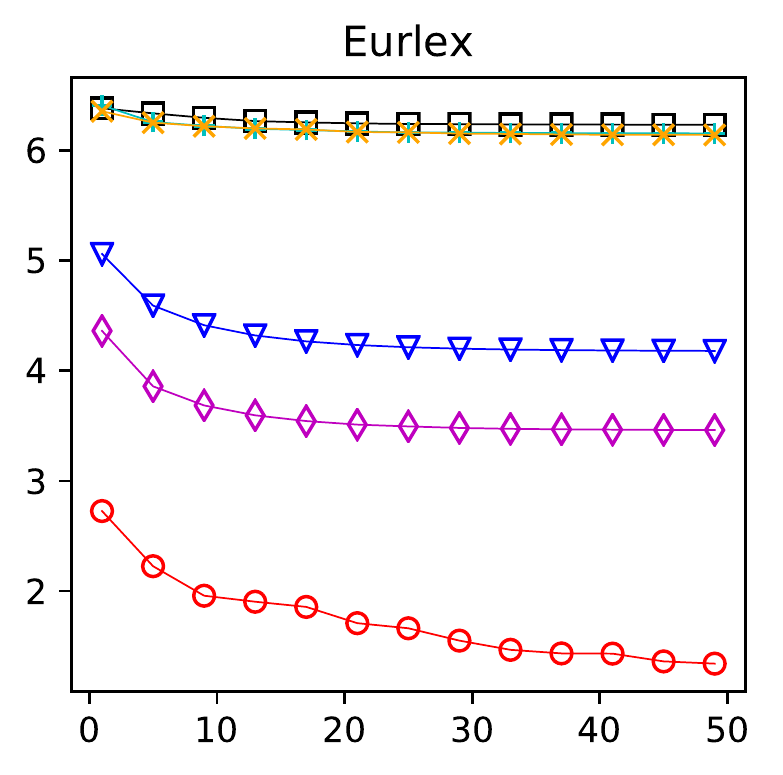}
\end{minipage}%
\hfill
\begin{minipage}{.24\textwidth}
  \centering
  \includegraphics[width=.99\linewidth]{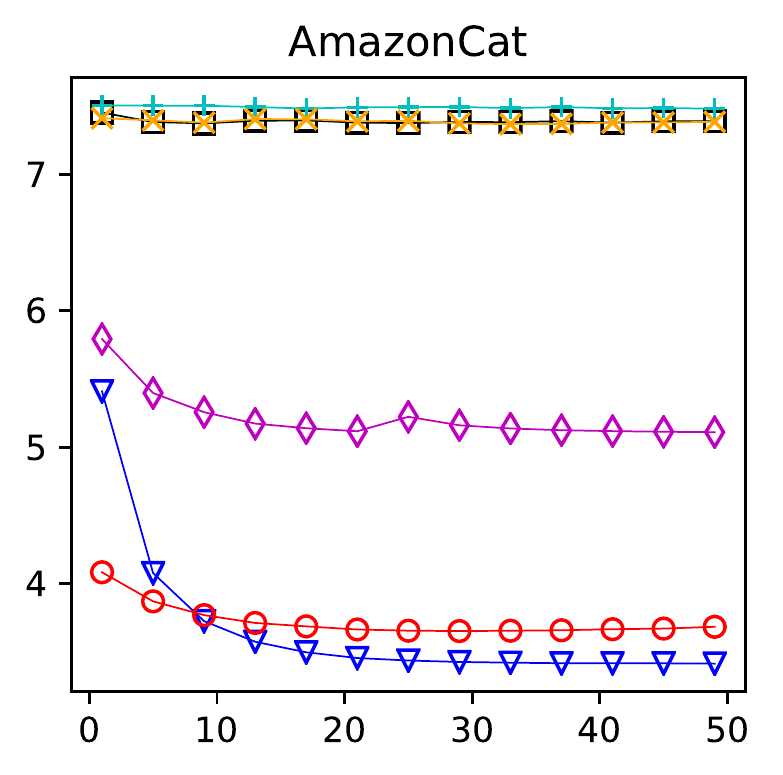}
\end{minipage}
\begin{minipage}{.24\textwidth}
  \centering
  \includegraphics[width=.99\linewidth]{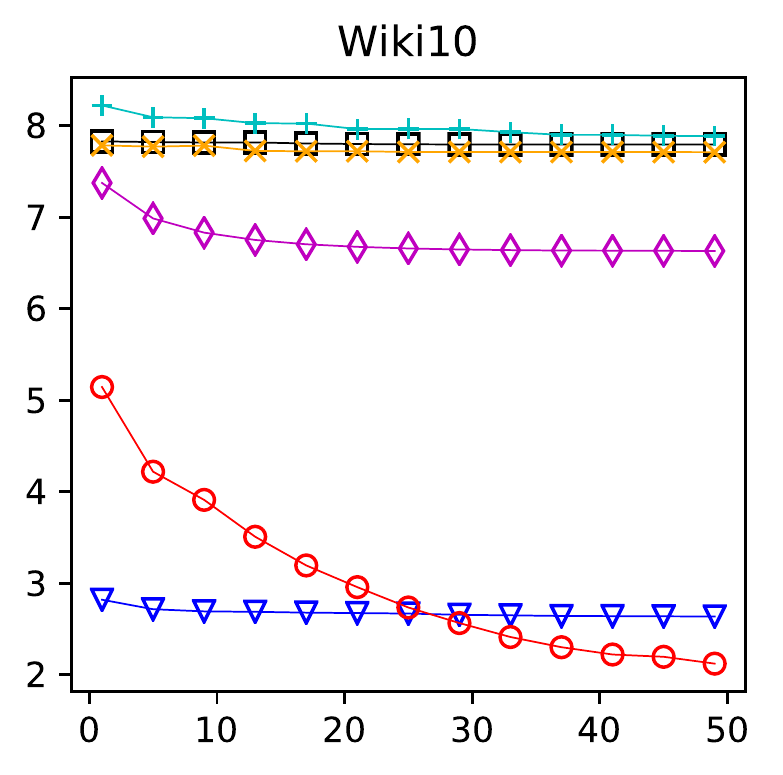}
\end{minipage}%
\hfill
\begin{minipage}{.24\textwidth}
  \centering
  \includegraphics[width=.99\linewidth]{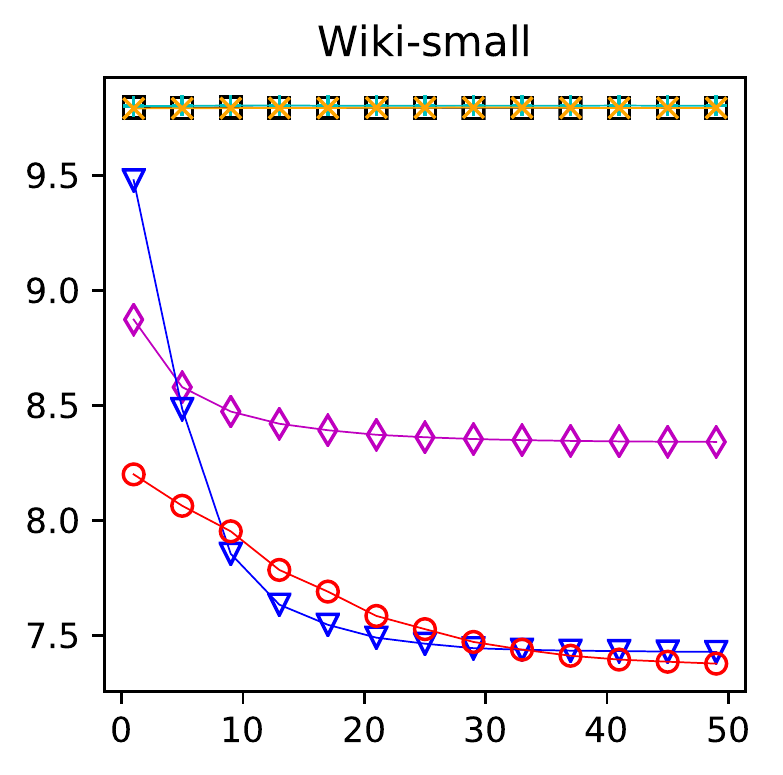}
\end{minipage}
\caption{The x-axis is the number of epochs and the y-axis is the log-loss from~(\ref{eq:original_log_likelihood}).}
\label{fig:comparison_prediction_log_loss}
\end{figure*}

\begin{table*}[h!]
\caption{Relative log-loss after 50 epochs. The values for each dataset are normalized by dividing by the corresponding Implicit log-loss. The algorithm with the lowest log-loss for each dataset is in bold.}
\label{tbl:results}
\vskip 0.15in
\begin{center}
\begin{small}
\begin{sc}
\begin{tabular}{ccccccc}
\toprule
Data set & OVE & NCE & IS & Vanilla & U-max & Implicit \\
\midrule
MNIST  & 5.25 & 5.55 & 5.26 & 1.31 & 1.40 &  \textbf{1.00}\\
Bibtex & 12.65 & 12.65 & 12.48 & 6.61 & 4.25 &    \textbf{1.00}\\
Delicious  & 1.77 & 1.78 & 1.76 & 1.16 & 1.03 &    \textbf{1.00}\\
Eurlex & 4.65 & 4.59 & 4.58 & 2.58 & 1.50 &   \textbf{1.00} \\
AmazonCat  & 2.01 & 2.03 & 2.00 & 1.39 & \textbf{0.93} &   1.00 \\
Wiki10 & 3.68 & 3.72 & 3.64 & 3.13 & 1.24 &   \textbf{1.00} \\
WikiSmall & 1.33 & 1.33 & 1.33 & 1.13 & 1.01 &   \textbf{1.00} \\
\midrule
Average & 4.48 & 4.52 & 4.44 & 2.47 & 1.62 &   \textbf{1.00} \\
\bottomrule
\\
\end{tabular}
\end{sc}
\end{small}
\end{center}
\vskip -0.1in
\end{table*}

\begin{figure*}[h!]
\centering
\begin{minipage}{.24\textwidth}
  \centering
   {\scriptsize \textsf{Learning rates}}\\
  \includegraphics[width=.35\linewidth]{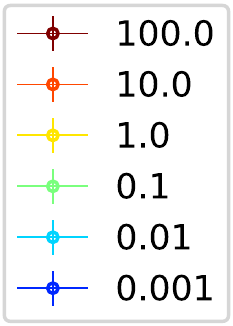}
\end{minipage}%
\hfill
\begin{minipage}{.24\textwidth}
  \centering
  \includegraphics[width=.99\linewidth]{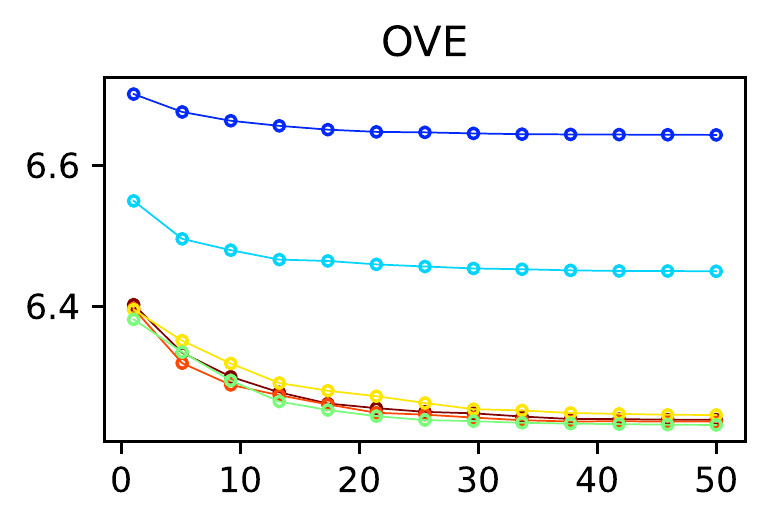}
\end{minipage}
\begin{minipage}{.24\textwidth}
  \centering
  \includegraphics[width=.99\linewidth]{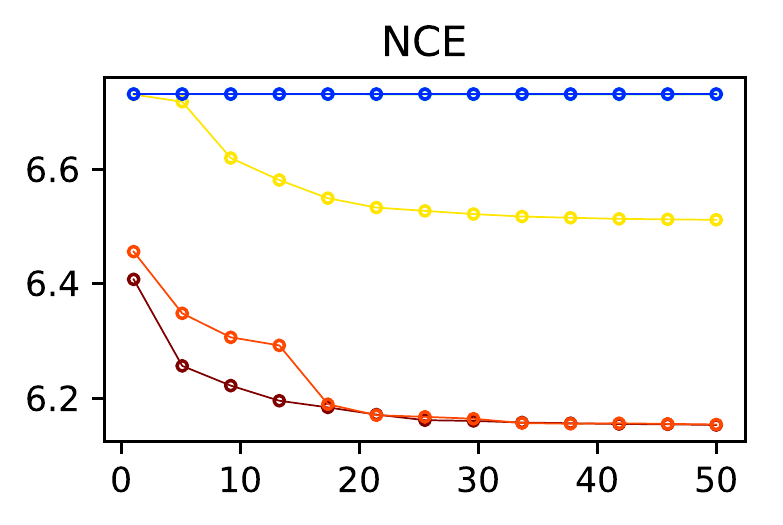}
\end{minipage}
\begin{minipage}{.24\textwidth}
  \centering
  \includegraphics[width=.99\linewidth]{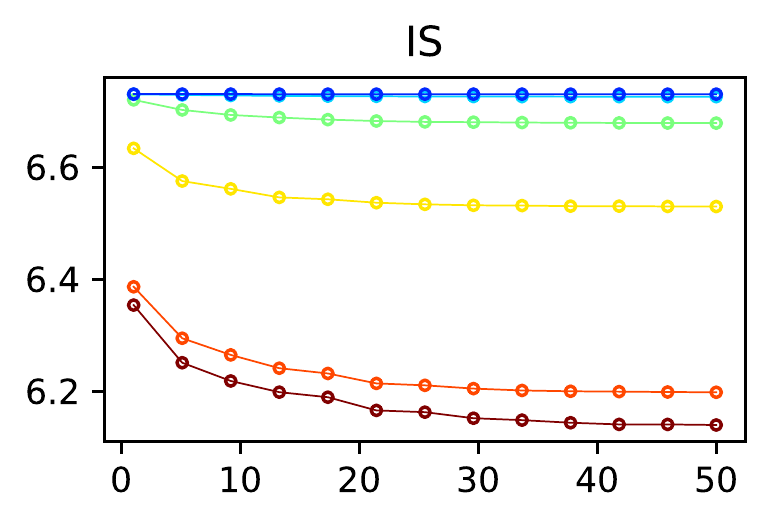}
\end{minipage}
\hfill
\begin{minipage}{.33\textwidth}
  \centering
  \includegraphics[width=.75\linewidth]{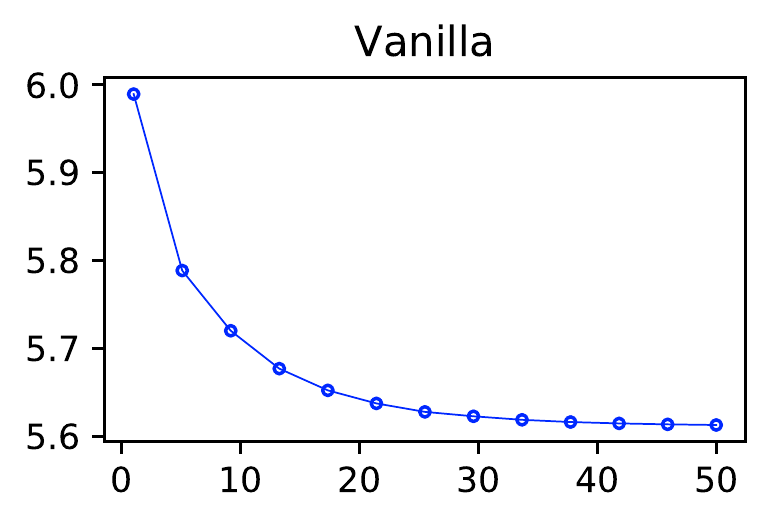}
\end{minipage}
\begin{minipage}{.33\textwidth}
  \centering
  \includegraphics[width=.75\linewidth]{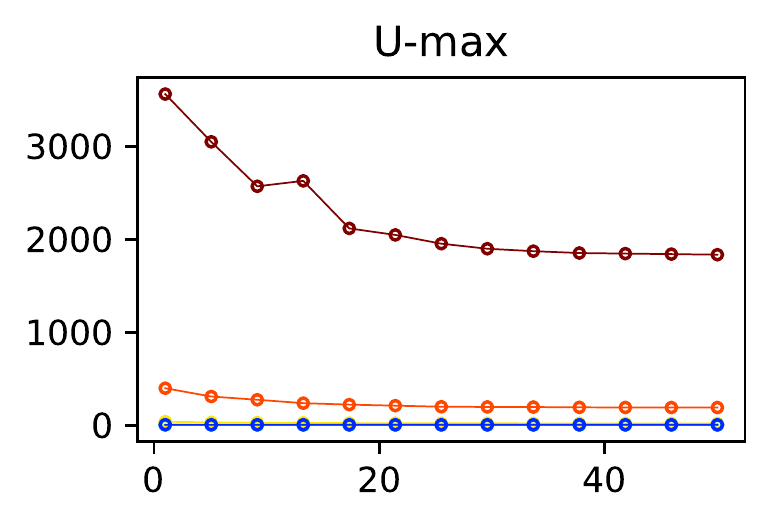}
\end{minipage}%
\hfill
\begin{minipage}{.33\textwidth}
  \centering
  \includegraphics[width=.75\linewidth]{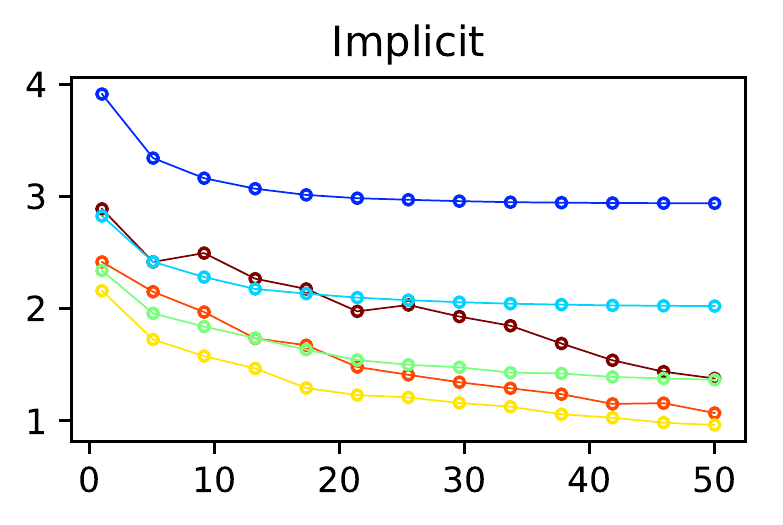}
\end{minipage}
\caption{Log-loss on Eurlex for different learning rates.}
\label{fig:learning_rate_loss}
\end{figure*}

\subsection{Results}\label{sec:experiment_results}
\noindent\textbf{Comparison to state-of-the-art.} Plots of the performance of the algorithms on each dataset are displayed in Figure~\ref{fig:comparison_prediction_log_loss} with the relative performance compared to Implicit SGD given in Table~\ref{tbl:results}.  

The Implicit SGD method has the best performance on all datasets but one. After just one epoch its performance is better than all of the state-of-the-art biased methods are after 50.
Not only does it converge faster in the first few epochs, it also converges to the optimal MLE (unlike the biased methods that prematurely plateau). 
On average after 50 epochs Implicit SGD's log-loss is a factor of 4.44 times lower than that of the biased methods. 

Out of the algorithms that sample more than one class per iteration, U-max's performance is the best. It is the only algorithm to outperform Implicit SGD on a dataset (AmazonCat). 
Vanilla SGD's performance is better than the previous state-of-the-art but is generally worse than U-max. The difference in performance between vanilla SGD and U-max can largely be explained by vanilla SGD requiring a smaller learning rate to avoid computational overflow.

The sensitivity of each method to the initial learning rate can be seen in Figure~\ref{fig:learning_rate_loss}, where the results of running each method on the Eurlex dataset with learning rates $\eta = 10^{0,\pm1,\pm2,-3}/N$ is presented. The results agree with those in Figure~\ref{fig:comparison_prediction_log_loss}, with Implicit SGD having the best performance for most learning rate settings. This is consistent with the theoretical results proving that Implicit SGD is robust to the learning rate \citep{ryu2014stochastic, toulis2015implicit}. In fact, Implicit SGD's worst performance is still better than the best performance all of the other algorithms.

For learning rates $\eta = 10^{1,2}/N$ the U-max log-loss is extremely large. This can be explained by Proposition~\ref{thm:Umax}, which does not guarantee convergence for U-max if the learning rate is too high. Vanilla SGD only has one line plotted, corresponding to the learning rate of $10^{-3}/N$, as for any high learning rate the algorithm suffered from computational overflow. The OVE, NCE and IS methods are very robust to the learning rate, which is perhaps why they have been so popular in the past.\\


\noindent\textbf{Comparison of double-sum formulations.} Figure~\ref{fig:double_sum_comparison} illustrates the performance on the Eurlex dataset of U-max using the proposed double-sum in~(\ref{eq:stochastic_f}) compared to U-max using the double-sum of \citet{raman2016ds} in~(\ref{eq:f_original_log_likelihood}). The proposed double-sum outperforms for all\footnote{The learning rates $\eta = 10^{1,2,3,4}/N$ are not displayed in the Figure~\ref{fig:double_sum_comparison} for visualization purposes. They have similar behavior as $\eta = 1.0/N$.} learning rates $\eta = 10^{0,\pm1,\pm2,\pm3,\pm4}/N$, with its $50^{th}$-epoch log-loss being $3.08$ times lower on average. This supports the argument from Section~\ref{sec:double_sum_formulation} that SGD methods applied to the proposed double-sum have smaller magnitude gradients and converge faster. Indeed, if the log-loss of vanilla SGD, U-max and Implicit SGD in Figure~\ref{fig:comparison_prediction_log_loss} and Table~\ref{tbl:results} were multiplied by 3.08 they would be roughly the same as OVE, NCE and IS. Thus our proposed double-sum formulation is crucial to the success of the U-max and Implicit SGD algorithms.

\begin{figure}[t]
\centering
\begin{minipage}{.39\columnwidth}
  \centering
  \includegraphics[width=.99\linewidth]{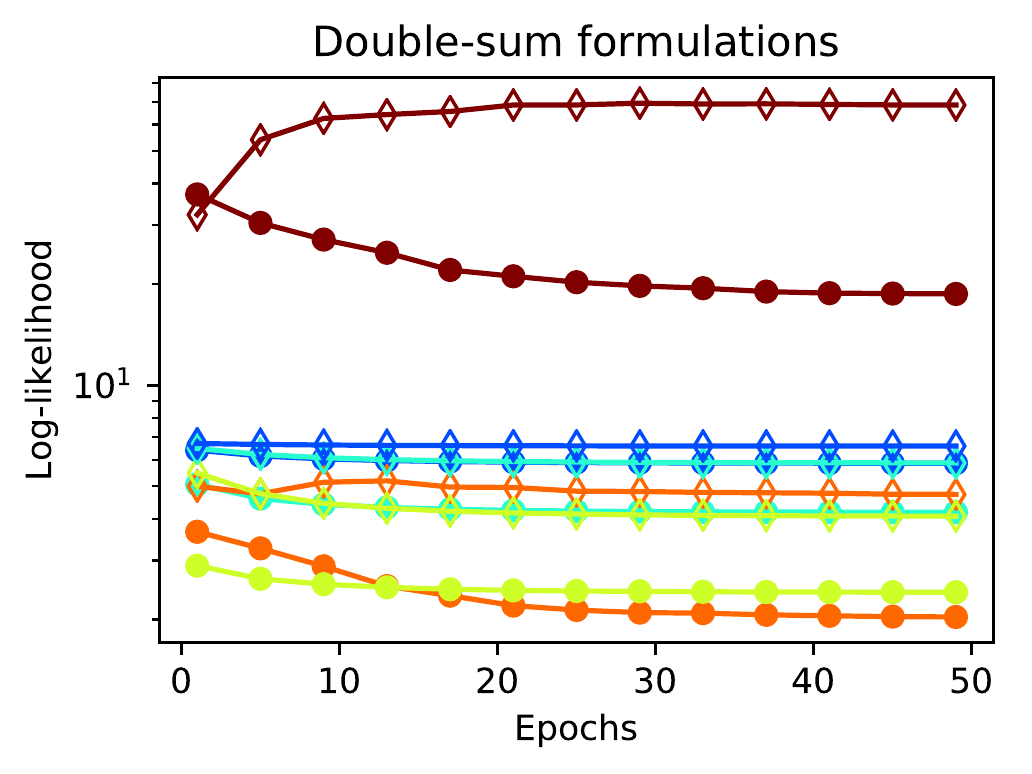}
\end{minipage}
\begin{minipage}{.23\columnwidth}
  \centering
  \includegraphics[width=.6\linewidth]{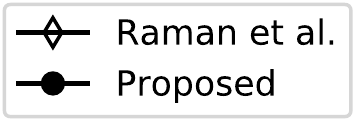}\\
  
  {\scriptsize \textsf{Learning rates}}\\
  \includegraphics[width=.45\linewidth]{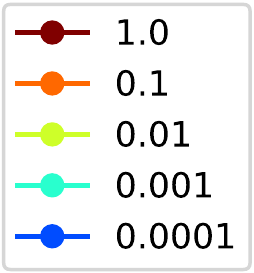}
\end{minipage}
\caption{Log-loss of U-max on Eurlex for different learning rates with our proposed double-sum formulation and that of \citet{raman2016ds}.}
\label{fig:double_sum_comparison}
\end{figure}

\section{Conclusion}\label{sec:conclusion}
In this paper we propose two unbiased robust algoritms for for
optimizing the softmax likelihood: Implicit SGD and U-max. These are the
\emph{first} unbiased algorithms that 
require only $O(D)$ computation per iteration, and no additional work at the
end of each epoch. 
Implicit SGD
can be efficiently implemented and clearly out-performs the previous
state-of-the-art on seven real world datasets.  
The result is a new method that enables optimizing the softmax for
extremely large number of samples 
and classes.

One limitation of the Implicit SGD method is that it is relatively slow if multiple datapoints are sampled each iteration or multiple inner-products can be efficiently computed (e.g. using GPUs). U-max should be the method of choice in such a setting. 

We only tested U-max and Implicit SGD 
 on 
the simple
softmax, but these methods can also be applied to any neural network where the final
layer is the softmax. Furthermore, applying these methods to word2vec type models, which
can be viewed as a softmax where both $x$ and $w$ are parameters to be
fit, might be provide a significant speed-up. 

\bibliographystyle{icml2018}
\bibliography{unbiased_softmax}

\clearpage
\appendix

\section{Comparison of double-sum formulations}\label{app:double_sum}
In Section~\ref{sec:double_sum_formulation} our double-sum formulation was compared to that of \citet{raman2016ds}. It was noted that the formulations only differ by a reparameterization $\bar{u}_i = u_i + x_i^\top w_{y_i}$, and an intuitive argument was given as to why our formulation leads to smaller magnitude gradients. Here we flesh out that argument and also explore different reparameterizations.

Let us introduce the set of parameterizations $v_i = \log(1+\sum_{k\neq y_i}e^{x_i^\top(w_k-w_{y_i})}) + \alpha x_i^\top w_{y_i}$ where $\alpha\in\mathbb{R}$. Our double-sum corresponds to $\alpha = 0$ while that of \citet{raman2016ds} to $\alpha = 1$. The question is, what is the optimal $\alpha$? The stochastic functions with $v_i$ are of the form
\begin{align*}
f_{ik}(v,W) =& N\left( v_i - \alpha x_i^\top w_{y_i} + e^{\alpha x_i^\top w_{y_i}-v_i}+(K-1)e^{x_i^\top (w_k - (1-\alpha)w_{y_i})-v_i})\right)
\end{align*}
where for notational simplicity we have set the ridge-regularization parameter $\mu=0$. The stochastic gradients are
\begin{align*}
\nabla_{w_k} f_{ik} &= N(K-1)e^{x_i^\top (w_k - (1-\alpha)w_{y_i})-v_i} x_i \\ 
\nabla_{w_{y_i}} f_{ik}&= N\left(-\alpha + \alpha e^{\alpha x_i^\top w_{y_i}-v_i}  - (1-\alpha)(K-1)e^{x_i^\top (w_k - (1-\alpha)w_{y_i})-v_i}\right) x_i \\ 
\nabla_{w_j} f_{ik}&= 0 \qquad \forall j\notin\{k,y_i\}\\ 
\nabla_{u_i} f_{ik}&=   N\left(1 - e^{\alpha x_i^\top w_{y_i}-v_i} - (K-1)e^{x_i^\top (w_k - (1-\alpha)w_{y_i})-v_i}\right).
\end{align*}

Let $x_i^\top w_{y_i} = x_i^\top \tilde{w}_{y_i} + \delta_i$ where $\tilde{w}_{y_i}$ is the old value of $w_{y_i}$ from the previous time datapoint $i$ was sampled. Let us write $v_i = \tilde{u}_i- \gamma_i+ \alpha(x_i^\top \tilde{w}_{y_i} - \epsilon_i)$, where $\gamma_i$ is the error between $\log(1+\sum_{k\neq y_i}e^{x_i^\top(\tilde{w}_k-\tilde{w}_{y_i})})$ and its estimate $\tilde{u}_i$, while $\epsilon_i$ is the error from estimating the $x_i^\top \tilde{w}_{y_i}$ term. The gradients become
\begin{align*}
\nabla_{w_k} f_{ik} &= N(K-1)\cdot e^{(1-\alpha)\delta_i + \alpha \epsilon_i - \gamma_i} \cdot e^{x_i^\top (w_k - \tilde{w}_{y_i})-\tilde{u}_i} x_i \\ 
\nabla_{w_{y_i}} f_{ik}&= N\left(-\alpha + \alpha e^{\alpha\delta_i + \alpha \epsilon_i - \gamma_i} \cdot e^{-\tilde{u}_i}  
- (1-\alpha)e^{(1-\alpha)\delta_i + \alpha \epsilon_i - \gamma_i} \cdot (K-1)e^{x_i^\top (w_k - \tilde{w}_{y_i})-\tilde{u}_i} \right) x_i \\ 
\nabla_{u_i} f_{ik}&=   N\left(1 - e^{\alpha\delta_i + \alpha \epsilon_i - \gamma_i} \cdot e^{-\tilde{u}_i} - e^{(1-\alpha)\delta_i + \alpha \epsilon_i - \gamma_i} \cdot(K-1) e^{x_i^\top (w_k - \tilde{w}_{y_i})-\tilde{u}_i} \right).
\end{align*}

The goal is for the variance of these stochastic gradients to be as small as possible. This may be achieved by setting $\alpha$ to decrease the effect of the noise factors $\delta_i$ and $\epsilon_i$. The noise $\epsilon_i$ always appears as $e^{\alpha \epsilon_i}$ and so it is best to have $\alpha= 0$ from this perspective. The noise $\delta_i$ appears as $e^{\alpha\delta_i}$,~$e^{(1-\alpha)\delta_i } \cdot(K-1) e^{x_i^\top (w_k - \tilde{w}_{y_i})}$ or $(1-\alpha)e^{(1-\alpha)\delta_i } \cdot(K-1) e^{x_i^\top (w_k - \tilde{w}_{y_i})}$, and so there is tension between setting $\alpha=0$ or $\alpha = 1$.

The optimal value of $\alpha$ clearly depends on the data and algorithm. If the noise $\epsilon$ is large and it is often the case that $x_i^\top \tilde{w}_{y_i} > x_i^\top \tilde{w}_k + \log(K-1)$ then $\alpha\approx 0$ is ideal, but if not then it is best for $\alpha\approx 1$. In Section~\ref{sec:experiments} we showed that for our datasets $\alpha=0$ yields better results than $\alpha=1$, although the optimal value of $\alpha$ is probably between $0$ and $1$. A future line of work is to develop methods to learn the optimal $\alpha$, perhaps dynamically per datapoint.

\section{Proof of variable bounds and strong
  convexity}\label{app:f_properties} 
We first establish that the optimal values of $u$ and $W$ are
bounded. 
Next, we show that within these bounds the objective is strongly convex
and its gradients are bounded.  
 

\begin{lemma}[\citet{raman2016ds}]\label{lemma:W_bound} The optimal value of $W$ is bounded as $\|W^\ast\|_2^2\leq B_W^2$ where $B_W^2 = \frac{2}{\mu}N \log(K)$.
\end{lemma}
\begin{proof}
\begin{align*}
-N\log(K) &= L(0) \leq L(W^\ast)\leq -\frac{\mu}{2}\|W^\ast\|_2^2
\end{align*}
Rearranging gives the desired result.
\end{proof}


\begin{lemma} \label{lemma:u_bound}The optimal value of $u_i$ is bounded as $u_i^\ast \leq B_u$ where $B_u = \log(1+(K-1)e^{2B_xB_w})$ and $B_x = \max_{i}\{\|x_i\|_2\}$
\end{lemma}
\begin{proof}
\begin{align*}
u_i^\ast &= \log(1+\sum_{k\neq y_i}e^{x_i^\top (w_k - w_{y_i})})\\
&\leq \log(1+\sum_{k\neq y_i}e^{\|x_i\|_2(\|w_k\|_2 + \|w_{y_i}\|_2)})\\
&\leq \log(1+\sum_{k\neq y_i}e^{2B_xB_w})\\
&= \log(1+(K-1)e^{2B_xB_w})
\end{align*}
\end{proof}


\begin{lemma}\label{lemma:strong_cvx} If $\|W\|_2^2\leq B_W^2$ and $u_i \leq B_u$ then $f(u,W)$ is strongly convex with convexity constant greater than or equal to $\min\{\exp(-B_u),\mu\}$.
\end{lemma}
\begin{proof}
Let us rewrite $f$ as
\begin{align*}
f(u, W) &= \sum_{i=1}^N u_i + e^{-u_i}+\sum_{k\neq y_i}e^{x_i^\top (w_k - w_{y_i})-u_i} + \frac{\mu}{2}\|W\|_2^2\\
&= \sum_{i=1}^N a_i^\top \theta + e^{-u_i}+\sum_{k\neq y_i}e^{b_{ik}^\top \theta} + \frac{\mu}{2}\|W\|_2^2.
\end{align*}
where $\theta = (u^\top,w_1^\top,...,w_k^\top)\in\mathbb{R}^{N+KD}$ with $a_i$ and $b_{ik}$ being appropriately defined. The Hessian of $f$ is
\begin{align*}
\nabla^2 f(\theta) &= \sum_{i=1}^N e^{-u_i}e_ie_i^\top+\sum_{k\neq y_i}e^{b_{ik}^\top \theta}b_{ik}b_{ik}^\top + \mu\cdot diag\{0_{N},1_{KD}\}
\end{align*}
where $e_i$ is the $i^{th}$ canonical basis vector, $0_{N}$ is an $N$-dimensional vector of zeros and $1_{KD}$ is a $KD$-dimensional vector of ones. It follows that
\begin{align*}
\nabla^2 f(\theta) &\succeq I \cdot \min\{\min_{0\leq u\leq B_u}\{e^{-u_i}\},\mu\}\\
&= I \cdot \min\{\exp(-B_u),\mu\}\\
&\succeq 0.
\end{align*}

\end{proof}


\begin{lemma}\label{lemma:bounded_gradients}
If $\|W\|_2^2\leq B_W^2$ and $u_i \leq B_u$ then the 2-norm of both the gradient of $f$ and each stochastic gradient $f_{ik}$ are bounded by
\begin{align*}
B_f&= N \max\{1,  e^{B_u} - 1\} + 2(N e^{B_u}B_x+ \mu \max_k\{\beta_k\}B_W).
\end{align*}

\end{lemma}
\begin{proof}
By Jensen's inequality
\begin{align*}
\max_{\|W\|_2^2\leq B_W^2, 0\leq u \leq B_u} \|\nabla f (u,W)\|_2 &= \max_{\|W\|_2^2\leq B_W^2, 0\leq u \leq B_u} \|\nabla \mathbb{E}_{ik} f_{ik}(u,W)\|_2\\
&\leq  \max_{\|W\|_2^2\leq B_W^2, 0\leq u \leq B_u} \mathbb{E}_{ik} \|\nabla f_{ik}(u,W)\|_2\\
&\leq \max_{\|W\|_2^2\leq B_W^2, 0\leq u \leq B_u}\max_{ik} \|\nabla f_{ik} (u,W)\|_2.
\end{align*}
Using the results from Lemmas~\ref{lemma:W_bound} and~\ref{lemma:u_bound} and the definition of $f_{ik}$ from~(\ref{eq:stochastic_f}),
\begin{align*}
\|\nabla_{u_i} f_{ik}(u,W)\|_2 &= \|N\left( 1 - e^{-u_i}-(K-1)e^{x_i^\top (w_k - w_{y_i})-u_i})\right)\|_2\\
&= N | 1 - e^{-u_i}(1+(K-1)e^{x_i^\top (w_k - w_{y_i})})|\\
&\leq N \max\{1,  (1+(K-1)e^{\|x_i\|_2 (\|w_k\|_2 + \|w_{y_i}\|_2)}) - 1\}\\
&\leq N \max\{1,  e^{B_u} - 1\}
\end{align*}
and for $j$ indexing either the sampled class $k\neq y_i$ or the true label $y_i$,
\begin{align*}
\|\nabla_{w_j} f_{ik}(u,W)\|_2 &= \|\pm N (K-1)e^{x_i^\top (w_k - w_{y_i})-u_i}x_i + \mu \beta_jw_j \|_2\\
&\leq N (K-1)e^{\|x_i\|_2 (\|w_k\|_2 + \|w_{y_i}\|_2)}\|x_i\|_2+ \mu \beta_j\|w_j \|_2\\
&\leq N e^{B_u}B_x+ \mu \max_k\{\beta_k\}B_W.
\end{align*}
Letting 
\begin{align*}
B_f &= N \max\{1,  e^{B_u} - 1\} + 2(N e^{B_u}B_x+ \mu \max_k\{\beta_k\}B_W)
\end{align*}
we have
\begin{align*}
\|\nabla f_{ik}(u,W)\|_2 &\leq  \|\nabla_{u_i} f_{ik}(u,W)\|_2 + \|\nabla_{w_k} f_{ik}(u,W)\|_2 + \|\nabla_{w_{y_i}} f_{ik}(u,W)\|_2= B_f.
\end{align*}
In conclusion:
\begin{align*}
\max_{\|W\|_2^2\leq B_W^2, 0\leq u \leq B_u} \|\nabla f (u,W)\|_2 \leq  \max_{\|W\|_2^2\leq B_W^2, u_i \leq B_u,}\max_{ik} \|\nabla f_{ik} (u,W)\|_2 \leq B_f.
\end{align*}
\end{proof}

\section{Stochastic Composition Optimization}\label{app:Mendi}
We can write the equation for $L(W)$ from~(\ref{eq:OVE_log_likelihood}) as (where we have set $\mu=0$ for notational simplicity),
\begin{align*}
L(W) &= -\sum_{i=1}^N \log(1 + \sum_{k\neq y_i}e^{x_i^\top (w_k-w_{y_i}}))\\
&= \mathbb{E}_i [h_i(\mathbb{E}_k[g_k(W)])]
\end{align*}
where $i\sim unif(\{1,...,N\})$, $k\sim unif(\{1,...,K\})$, $h_i(v)\in\mathbb{R}$, $g_k(W)\in\mathbb{R}^N$ and
\begin{align*}
h_i(v) &= -N \log(1 +e_i^\top v)\\
[g_k(W)]_i &= \begin{cases}
Ke^{x_i^\top (w_k-w_{y_i})} &\mbox{ if }k\neq y_i\\
0 &\mbox{ otherwise}
\end{cases}.
\end{align*}
Here $e_i^\top v = v_i\in\mathbb{R}$ is a variable that is explicitly kept track of with $v_i\approx \mathbb{E}_k[g_k(W)]_i = \sum_{k\neq y_i}e^{x_i^\top (w_k-w_{y_i}})$ (with exact equality in the limit as $t\to\infty$). Clearly $v_i$ in stochastic composition optimization has a similar role as $u_i$ has in our formulation for $f$ in~(\ref{eq:f}).

If $i,k$ are sampled  with $k\neq y_i$ in stochastic composition optimization then the updates are of the form \citep{wang2016accelerating}
\begin{align*}
w_{y_i} &= w_{y_i} + \eta_tNK\frac{ e^{x_i^\top (z_k - z_{y_i})}}{1+v_i}x_i\\
w_k &= w_k - \eta_tNK\frac{ e^{x_i^\top (z_k - z_{y_i})}}{1+v_i}x_i,
\end{align*}
where $z_k$ is a smoothed value of $w_k$. These updates have the same numerical instability issues as vanilla SGD on $f$ in~(\ref{eq:f}): it is possible that $\frac{e^{x_i^\top z_k}}{1+v_i} \gg 1$ where ideally we should have $0 \leq \frac{e^{x_i^\top z_k}}{1+v_i} \leq 1$.

\section{Proof of general Implicit SGD gradient bound}\label{app:general_implicit_step_bound}
\begin{proof}[Proof of Proposition~\ref{prop:step_size}]
Let $f(\theta, \xi)$ be $m$-strongly convex for all~$\xi$. The vanilla SGD step size is $\eta_t\|\nabla f(\theta^{(t)}, \xi_t)\|_2$ where $\eta_t$ is the learning rate for the $t^{th}$ iteration. The Implicit SGD step size is $\eta_t\|\nabla f(\theta^{(t+1)}, \xi_t)\|_2$ where $\theta^{(t+1)}$ satisfies $\theta^{(t+1)} = \theta^{(t)} - \eta_t\nabla f(\theta^{(t+1)}, \xi_t)$. Rearranging, $\nabla f(\theta^{(t+1)}, \xi_t) = (\theta^{(t)} - \theta^{(t+1)})/\eta_t$ and so it must be the case that $\nabla f(\theta^{(t+1)}, \xi_t)^\top (\theta^{(t)} - \theta^{(t+1)}) = \|\nabla f(\theta^{(t+1)}, \xi_t)\|_2 \|\theta^{(t)} - \theta^{(t+1)}\|_2$.


Our desired result follows:
\begin{align*}
\|\nabla f(\theta^{(t)}, \xi_t)\|_2 &\geq \frac{\nabla f(\theta^{(t)})^\top (\theta^{(t)} - \theta^{(t+1)})}{\|\theta^{(t)} - \theta^{(t+1)}\|_2}\\
&\geq \frac{\nabla f(\theta^{(t+1)})^\top (\theta^{(t)} - \theta^{(t+1)})+m\|\theta^{(t)} - \theta^{(t+1)}\|_2^2}{\|\theta^{(t)} - \theta^{(t+1)}\|_2}\\
&= \frac{\|\nabla f(\theta^{(t+1)})\|_2 \|\theta^{(t)} - \theta^{(t+1)}\|_2+m\|\theta^{(t)} - \theta^{(t+1)}\|_2^2}{\|\theta^{(t)} - \theta^{(t+1)}\|_2}\\
&= \|\nabla f(\theta^{(t+1)})\|_2 +m\|\theta^{(t)} - \theta^{(t+1)}\|_2
\end{align*}
where the first inequality is by Cauchy-Schwarz and the second inequality by strong convexity.

\end{proof}

\section{Update equations for Implicit SGD}\label{app:Implicit_SGD}
In this section we will derive the updates for Implicit SGD. We will first consider the simplest case where only one datapoint $(x_i,y_i)$ and a single class is sampled in each iteration. Later we will derive the updates for when multiple classes are sampled, and finally when both multiple classes and multiple datapoints are sampled.

\subsection{Single datapoint, single class}\label{app:Implicit_SGD_simpled}
Equation~(\ref{eq:stochastic_f}) for the stochastic gradient with a single datapoint, single sampled class is
\begin{align*}
f_{ik}(u,W) & = N(u_i + e^{-u_i}+(K-1)e^{x_i^\top (w_k - w_{y_i})-u_i}) + \frac{\mu}{2}(\beta_{y_i} \|w_{y_i}\|_2^2+  \beta_k \|w_k\|_2^2).
\end{align*}
The Implicit SGD update corresponds to finding the variables optimizing
\begin{align*}
 \min_{u,W}\left\{ 2\eta f_{ik}(u,W) + \|u - \tilde{u}\|_2^2+\|W - \tilde{W}\|_2^2 \right\},
\end{align*}
where $\eta$ is the learning rate and the tilde refers to the value of the old iterate \citep[Eq. 6]{toulis2016towards}. Since $f_{ik}$ is only a function of $u_i, w_k, w_{y_i}$, we have the optimal $w_j=\tilde{w}_j$ for $j\notin\{k,y_i\}$ and $u_j=\tilde{u}_j$ for $j\neq i$. The optimization reduces to 
\begin{align*}
&\min_{u_i, w_k, w_{y_i}}\left\{2\eta f_{ik}(u_i,w_k, w_{y_i}) + (u_i - \tilde{u}_i)^2+\|w_{y_i} - \tilde{w}_{y_i}\|_2^2+\|w_k - \tilde{w}_k\|_2^2\right\}  \\
&= \min_{u_i, w_k, w_{y_i}}\bigg\{2\eta N (u_i + e^{-u_i}+(K-1)e^{x_i^\top (w_k - w_{y_i})-u_i}) + \eta\mu(\beta_{y_i} \|w_{y_i}\|_2^2+  \beta_k \|w_k\|_2^2)\\
&\qquad\qquad\qquad + (u_i - \tilde{u}_i)^2+\|w_{y_i} - \tilde{w}_{y_i}\|_2^2+\|w_k - \tilde{w}_k\|_2^2\bigg\}.
\end{align*}

\noindent\textbf{Solving for $w_k, w_{y_i}$ with auxiliary variable $b$}\\
Much of the difficulty in optimizing this equation comes from the interaction between the $e^{x_i^\top (w_k - w_{y_i})-u_i}$ term and the $\|\cdot\|_2^2$ terms. To isolate this interaction we introduce an auxiliary variable $b = x_i^\top (w_k - w_{y_i})$ and rewrite the optimization problem as
\begin{align*}
&\min_{u_i, b}\bigg\{2\eta N (u_i + e^{-u_i}+(K-1)e^{b-u_i}) + (u_i - \tilde{u}_i)^2 \\
&\qquad\quad + \min_{w_k, w_{y_i}}\left\{\eta\mu(\beta_{y_i} \|w_{y_i}\|_2^2+  \beta_k \|w_k\|_2^2)+\|w_{y_i} - \tilde{w}_{y_i}\|_2^2+\|w_k - \tilde{w}_k\|_2^2:~b = x_i^\top (w_k - w_{y_i})\right\}\bigg\}.
\end{align*}
The inner optimization problem over $w_k, w_{y_i}$ is a quadratic program with linear constraints. Taking the dual and solving yields
\begin{align}
w_k&= \frac{\tilde{w}_k}{1+\eta\mu\beta_k} - \gamma_i\frac{x_i^\top(\frac{\tilde{w}_k}{1+\eta\mu\beta_k} - \frac{\tilde{w}_{y_i}}{1+\eta\mu\beta_{y_i}}) - b}{1+\eta\mu\beta_k} x_i \nonumber\\ 
w_{y_i} &= \frac{\tilde{w}_{y_i}}{1+\eta\mu\beta_{y_i}} + \gamma_i\frac{x_i^\top(\frac{\tilde{w}_k}{1+\eta\mu\beta_k} - \frac{\tilde{w}_{y_i}}{1+\eta\mu\beta_{y_i}}) - b}{1+\eta\mu\beta_{y_i}} x_i \label{eq:implicit_single_class_w_update}
\end{align}
where 
\begin{equation}\label{eq:gamma}
\gamma_i = \frac{1}{\|x_i\|_2^2((1+\eta\mu\beta_k)^{-1} + (1+\eta\mu\beta_{y_i})^{-1})}
\end{equation}
Substituting in the solution for $w_k, w_{y_i}$ and dropping constant terms, the optimization problem reduces to 
\begin{equation}
\min_{u_i, b}\bigg\{2\eta N (u_i + e^{-u_i}+(K-1)e^{b-u_i}) + (u_i - \tilde{u}_i)^2 + \left(b - x_i^\top(\frac{\tilde{w}_k}{1+\eta\mu\beta_k} - \frac{\tilde{w}_{y_i}}{1+\eta\mu\beta_{y_i}})\right)^2 \gamma_i\bigg\}. \label{eq:implicit_sgd_single_class}
\end{equation}

We'll approach this optimization problem by first solving for $b$ as a function of $u_i$ and then optimize over $u_i$. Once the optimal value of $u_i$ has been found, we can calculate the corresponding optimal value of $b$. Finally, substituting $b$ into~(\ref{eq:implicit_single_class_w_update}) will give us our updated value of $W$.\\

\noindent\textbf{Solving for $b$}\\
We solve for $b$ by setting its derivative equal to zero in~(\ref{eq:implicit_sgd_single_class})
\begin{equation*}
0 = 2\eta N (K-1)e^{b-u_i} +  2\left(b - x_i^\top(\frac{\tilde{w}_k}{1+\eta\mu\beta_k} - \frac{\tilde{w}_{y_i}}{1+\eta\mu\beta_{y_i}})\right) \gamma_i.
\end{equation*}
Letting $a = x_i^\top(\frac{\tilde{w}_k}{1+\eta\mu\beta_k} - \frac{\tilde{w}_{y_i}}{1+\eta\mu\beta_{y_i}}) - b$ and using simple algebra yields
\begin{equation}
ae^{a}=   \eta N(K-1)\gamma_i^{-1}e^{x_i^\top(\frac{\tilde{w}_k}{1+\eta\mu\beta_k} - \frac{\tilde{w}_{y_i}}{1+\eta\mu\beta_{y_i}})-u_i}.\label{eq:a}
\end{equation}
The solution for $a$ can be written in terms of the principle branch of the Lambert-W function, $P$,
\begin{equation}
a(u_i) = P\left(\eta N(K-1)\gamma_i^{-1}e^{x_i^\top(\frac{\tilde{w}_k}{1+\eta\mu\beta_k} - \frac{\tilde{w}_{y_i}}{1+\eta\mu\beta_{y_i}})-u_i}\right).\label{eq:a_u_function}
\end{equation}
The optimal value of $b$ given $u_i$ is therefore $b(u_i) = x_i^\top(\frac{\tilde{w}_k}{1+\eta\mu\beta_k} - \frac{\tilde{w}_{y_i}}{1+\eta\mu\beta_{y_i}}) - a(u_i)$.\\

\noindent\textbf{Bisection method for $u_i$}\\
Substituting $b(u_i) = x_i^\top(\frac{\tilde{w}_k}{1+\eta\mu\beta_k} - \frac{\tilde{w}_{y_i}}{1+\eta\mu\beta_{y_i}}) - a(u_i)$ into~(\ref{eq:implicit_sgd_single_class}), we now only need minimize over $u_i$:
\begin{align}
&\min_{u_i}\bigg\{2\eta N (u_i + e^{-u_i}+(K-1)e^{x_i^\top(\frac{\tilde{w}_k}{1+\eta\mu\beta_k} - \frac{\tilde{w}_{y_i}}{1+\eta\mu\beta_{y_i}}) - a(u_i)-u_i}) + (u_i - \tilde{u}_i)^2 + a(u_i)^2\gamma_i\bigg\}\label{eq:min_with_a}\\
&=\min_{u_i}\bigg\{2\eta N u_i + 2\eta Ne^{-u_i}+2\gamma_i a(u_i) + (u_i - \tilde{u}_i)^2 + a(u_i)^2\gamma_i \bigg\}\nonumber\\
&=\min_{u_i}\bigg\{2\eta N u_i + 2\eta Ne^{-u_i} + (u_i - \tilde{u}_i)^2+\gamma_i a(u_i)(2+ a(u_i))\bigg\}\label{eq:implicit_sgd_single_class_u}
\end{align}
where and we used the fact that $e^{-P(z)}=P(z)/z$ to simplify the $(K-1)e^{x_i^\top(\frac{\tilde{w}_k}{1+\eta\mu\beta_k} - \frac{\tilde{w}_{y_i}}{1+\eta\mu\beta_{y_i}}) - a(u_i)-u_i}$ term. The derivative in~(\ref{eq:implicit_sgd_single_class_u}) with respect to $u_i$ is 
\begin{align}
& \partial_{u_i}\bigg\{ 2\eta N u_i + 2\eta Ne^{-u_i} + (u_i - \tilde{u}_i)^2+\gamma_i a(u_i)(2+ a(u_i)) \bigg\}\nonumber\\
&= 2\eta N - 2\eta Ne^{-u_i}+ 2(u_i - \tilde{u}_i)+2 \gamma_i (1+ a(u_i))\partial_{u_i} a(u_i)\nonumber\\
&= 2\eta N - 2\eta Ne^{-u_i}+ 2(u_i - \tilde{u}_i)-2 \gamma_i a(u_i) \label{eq:u_derivative}
\end{align}
where we used the fact that $\partial_z P(z) = \frac{P(z)}{z(1+P(z))}$ to work out that $\partial_{u_i} a(u_i) = -\frac{a(u_i)}{1+a(u_i)}$.

We can solve for $u_i$ using a bisection method. Below we show how to calculate the initial lower and upper bounds of the bisection interval and prove that the size of the interval is bounded (which ensures fast convergence). The initial lower and upper bounds we use depends on the derivative in~(\ref{eq:u_derivative}) at $u_i = \tilde{u}_i$. In deriving the bounds we will use $u_i'$ to denote the optimal value of $u_i$ and $a'$ to denote the optimal value of $a$.\\

\noindent\textbf{Case: $u_i' > \tilde{u}_i$}\\
If the derivative is negative then $u_i'$ is lower bounded by $\tilde{u}_i$. An upper bound on $u_i'$ can be derived from~(\ref{eq:min_with_a}):
\begin{align*}
u_i' &= \argmin_{u_i}\bigg\{2\eta N (u_i + e^{-u_i}+(K-1)e^{x_i^\top(\frac{\tilde{w}_k}{1+\eta\mu\beta_k} - \frac{\tilde{w}_{y_i}}{1+\eta\mu\beta_{y_i}}) - a'-u_i}) + (u_i - \tilde{u}_i)^2 \bigg\}\\
&\leq \argmin_{u_i}\bigg\{2\eta N (u_i + e^{-u_i}+(K-1)e^{x_i^\top(\frac{\tilde{w}_k}{1+\eta\mu\beta_k} - \frac{\tilde{w}_{y_i}}{1+\eta\mu\beta_{y_i}}) -u_i}) + (u_i - \tilde{u}_i)^2\bigg\}\\
&= \tilde{u}_i + P (\eta Ne^{\eta N-\tilde{u}_i}(1+(K-1)e^{x_i^\top(\frac{\tilde{w}_k}{1+\eta\mu\beta_k} - \frac{\tilde{w}_{y_i}}{1+\eta\mu\beta_{y_i}})})) - \eta N
\end{align*}
where in the inequality we set $a'=0$, since the minimal value of $u_i$ is monotonically decreasing in $a'$. This bound should be used in the bisection method, but for ease of analysis we can weaken the bound:
\begin{align*}
u_i' &\leq \argmin_{u_i}\bigg\{2\eta N (u_i + e^{-u_i}+(K-1)e^{x_i^\top(\frac{\tilde{w}_k}{1+\eta\mu\beta_k} - \frac{\tilde{w}_{y_i}}{1+\eta\mu\beta_{y_i}})}e^{-u_i})\bigg\}\\
&= \log(1+(K-1)e^{x_i^\top(\frac{\tilde{w}_k}{1+\eta\mu\beta_k} - \frac{\tilde{w}_{y_i}}{1+\eta\mu\beta_{y_i}})}).
\end{align*}
where we used the assumption that $u_i'$ is lower bounded by~$\tilde{u}_i$ to remove the $(u_i - \tilde{u}_i)^2$ term. 
Thus $\tilde{u}_i\leq u_i'\leq \log(1+(K-1)e^{x_i^\top(\frac{\tilde{w}_k}{1+\eta\mu\beta_k} - \frac{\tilde{w}_{y_i}}{1+\eta\mu\beta_{y_i}})})$. If $(K-1)e^{x_i^\top(\frac{\tilde{w}_k}{1+\eta\mu\beta_k} - \frac{\tilde{w}_{y_i}}{1+\eta\mu\beta_{y_i}})}\leq 1$ then the size of the bounding interval must be less than $\log(2)$, since $\tilde{u}_i\geq 0$. Otherwise the gap must be at most $\log(2(K-1)e^{x_i^\top(\frac{\tilde{w}_k}{1+\eta\mu\beta_k} - \frac{\tilde{w}_{y_i}}{1+\eta\mu\beta_{y_i}})}) - \tilde{u}_i = \log(2(K-1)) + x_i^\top(\frac{\tilde{w}_k}{1+\eta\mu\beta_k} - \frac{\tilde{w}_{y_i}}{1+\eta\mu\beta_{y_i}}) - \tilde{u}_i$. Either way, the size of the interval is upper bounded by $\log(2K) + |x_i^\top(\frac{\tilde{w}_k}{1+\eta\mu\beta_k} - \frac{\tilde{w}_{y_i}}{1+\eta\mu\beta_{y_i}}) - \tilde{u}_i|$.\\

\noindent\textbf{Case: $u_i' < \tilde{u}_i$}\\
Now let us consider if the derivative in~(\ref{eq:u_derivative}) is positive at $u_i = \tilde{u}_i$. Then $u_i'$ is upper bounded by $\tilde{u}_i$. We can lower bound $u_i'$ by:
\begin{align}
u_i' &= \argmin_{u_i}\bigg\{2\eta N (u_i + e^{-u_i}+(K-1)e^{x_i^\top(\frac{\tilde{w}_k}{1+\eta\mu\beta_k} - \frac{\tilde{w}_{y_i}}{1+\eta\mu\beta_{y_i}})}e^{-a'-u_i}) + (u_i - \tilde{u}_i)^2\bigg\}\label{eq:u_lower_bound_strong}\\
&\geq \argmin_{u_i}\bigg\{u_i + e^{-u_i}+(K-1)e^{x_i^\top(\frac{\tilde{w}_k}{1+\eta\mu\beta_k} - \frac{\tilde{w}_{y_i}}{1+\eta\mu\beta_{y_i}})}e^{-a'-u_i}\bigg\}\nonumber\\
&= \log(1+(K-1)\exp(x_i^\top(\frac{\tilde{w}_k}{1+\eta\mu\beta_k} - \frac{\tilde{w}_{y_i}}{1+\eta\mu\beta_{y_i}})-a'))\nonumber\\
&\geq \log(K-1) + x_i^\top(\frac{\tilde{w}_k}{1+\eta\mu\beta_k} - \frac{\tilde{w}_{y_i}}{1+\eta\mu\beta_{y_i}})-a'\label{eq:u_lower_bound}
\end{align}
where the first inequality comes dropping the $(u_i - \tilde{u}_i)^2$ term due to the assumption that $u_i' < \tilde{u}_i$ and the second inequality is from the monotonicity of the log function. Recall~(\ref{eq:a}),
\begin{equation*}
a'e^{a'}=   \eta N(K-1)\gamma_i^{-1}e^{x_i^\top(\frac{\tilde{w}_k}{1+\eta\mu\beta_k} - \frac{\tilde{w}_{y_i}}{1+\eta\mu\beta_{y_i}})-u_i'}.
\end{equation*}
We can upper bound $a'$ using the lower bound on $u'$:
\begin{align}
a' &= e^{-a'}a'e^{a'}\nonumber\\
&= e^{-a'}\eta N(K-1)\gamma_i^{-1}e^{x_i^\top(\frac{\tilde{w}_k}{1+\eta\mu\beta_k} - \frac{\tilde{w}_{y_i}}{1+\eta\mu\beta_{y_i}})-u_i'}\nonumber\\
&\leq e^{-a'}\eta N(K-1)\gamma_i^{-1}e^{x_i^\top(\frac{\tilde{w}_k}{1+\eta\mu\beta_k} - \frac{\tilde{w}_{y_i}}{1+\eta\mu\beta_{y_i}})-(\log(K-1) + x_i^\top(\frac{\tilde{w}_k}{1+\eta\mu\beta_k} - \frac{\tilde{w}_{y_i}}{1+\eta\mu\beta_{y_i}})-a')}\nonumber\\
&= \eta N\gamma_i^{-1}\label{eq:a_dash_bound}
\end{align}
Substituting this upper bound for $a'$ into~(\ref{eq:u_lower_bound_strong}) and solving yields a lower bound on $u_i'$,
\begin{align*}
u_i' &\geq  \tilde{u}_i + P (\eta Ne^{\eta N-\tilde{u}_i}(1+(K-1)e^{x_i^\top(\frac{\tilde{w}_k}{1+\eta\mu\beta_k} - \frac{\tilde{w}_{y_i}}{1+\eta\mu\beta_{y_i}})-\eta N\gamma_i^{-1}})) - \eta N .
\end{align*}

Again this bound should be used in the bisection method, but for ease of analysis we can weaken the bound by instead substituting the bound for $a'$ into~(\ref{eq:u_lower_bound}) which yields:
\begin{align*}
u_i' &\geq  \log(K-1) + x_i^\top(\frac{\tilde{w}_k}{1+\eta\mu\beta_k} - \frac{\tilde{w}_{y_i}}{1+\eta\mu\beta_{y_i}})-\eta N\gamma_i^{-1}.
\end{align*}
Thus $\log(K-1) + x_i^\top(\frac{\tilde{w}_k}{1+\eta\mu\beta_k} - \frac{\tilde{w}_{y_i}}{1+\eta\mu\beta_{y_i}})- \eta N\gamma_i^{-1} \leq u_i'\leq \tilde{u}_i$. The size of the bisection method interval is upper bounded by $\tilde{u}_i - x_i^\top(\frac{\tilde{w}_k}{1+\eta\mu\beta_k} - \frac{\tilde{w}_{y_i}}{1+\eta\mu\beta_{y_i}})+\eta N\gamma_i^{-1}-\log(K-1) $.

In summary, for both signs of the derivative in~(\ref{eq:u_derivative}) at $u_i = \tilde{u}_i$ we are able to lower and upper bound the optimal value of $u_i$ such that interval between the bounds is at most $|\tilde{u}_i - x_i^\top(\frac{\tilde{w}_k}{1+\eta\mu\beta_k} - \frac{\tilde{w}_{y_i}}{1+\eta\mu\beta_{y_i}})|+\eta N\gamma_i^{-1}+\log(2K)$. This allows us to perform the bisection method where for $\epsilon>0$ level accuracy we require only $\log_2(\epsilon^{-1}) + \log_2(|\tilde{u}_i - x_i^\top(\frac{\tilde{w}_k}{1+\eta\mu\beta_k} - \frac{\tilde{w}_{y_i}}{1+\eta\mu\beta_{y_i}})| + \eta N\gamma_i^{-1}+\log(2K))$ function evaluations. In practice we use Brent's method as the optimization routine, which is faster than the simple bisection method. The pseudocode of the entire method is displayed in Algorithm~\ref{alg:isgd}.

\begin{algorithm}[t]
   \caption{Implicit SGD with one datapoint and class sampled each iteration}
   \label{alg:isgd}
\begin{algorithmic}
      \STATE {\bfseries Input:} Data $\mathcal{D} = \{(y_i, x_i)\}_{i=1}^N$, number of iterations~$T$, learning rate $\eta_t$, threshold $\delta>0$, regularization constants $\beta$ and $\gamma$ from (\ref{eq:gamma}), principle Lambert-W function $P$, initial $u,W$.
   \STATE {\bfseries Ouput:} $W$
   \vspace{0.5cm}

   \FOR{$t=1$ {\bfseries to} $T$}
   \STATE  \texttt{Sample datapoint and classes}
   \STATE $i \sim unif(\{1,...,N\})$
   \STATE $k \sim unif(\{1,...,K\}-\{y_i\})$
   
   \vspace{0.5cm}
   \STATE   \texttt{Calculate gradient at $u_i=\tilde{u}_i$}
   \STATE $g \gets 2\eta_t N - 2\eta_t Ne^{-u_i} -2 \gamma_i P(\eta_t N(K-1)\gamma_i^{-1}e^{x_i^\top(\frac{w_k}{1+\eta_t\mu\beta_k} - \frac{w_{y_i}}{1+\eta_t\mu\beta_{y_i}})-u_i}) $
   
   \vspace{0.5cm}
   \STATE   \texttt{Calculate lower and upper bounds on $u_i$}
   \IF{$g <0$}
   \STATE $(b_l, b_u) \gets (u_i, ~u_i+ P (\eta_t Ne^{\eta_t N-u_i}(1+(K-1)e^{x_i^\top(\frac{w_k}{1+\eta_t\mu\beta_k} - \frac{w_{y_i}}{1+\eta_t\mu\beta_{y_i}})})) - \eta_t N)$
   \ELSIF{$g > 0$}
   \STATE $(b_l, b_u) \gets (u_i + P (\eta_t Ne^{\eta_t N-u_i}(1+(K-1)e^{x_i^\top(\frac{\tilde{w}_k}{1+\eta_t\mu\beta_k} - \frac{\tilde{w}_{y_i}}{1+\eta_t\mu\beta_{y_i}})-\eta_t N\gamma_i^{-1}})) - \eta_t N, ~u_i) $
   \ELSIF{$g = 0$}
   \STATE $(b_l, b_u) \gets (u_i, ~u_i) $
   \ENDIF
   
   \vspace{0.5cm}
   \STATE   \texttt{Optimize $u_i$ using Brent's method with bounds $b_l,b_u$ and gradient $g(u)$}
   \STATE $u_i \gets \texttt{Brents}(b_l, b_u, g(u) = 2\eta_t N - 2\eta_t Ne^{-u}+ 2(u - u_i)-2 \gamma_i P(\eta_t N(K-1)\gamma_i^{-1}e^{x_i^\top(\frac{w_k}{1+\eta_t\mu\beta_k} - \frac{w_{y_i}}{1+\eta_t\mu\beta_{y_i}})-u})) $

   \vspace{0.5cm}
   \STATE  \texttt{Update $w$}
    \STATE $w_{k_j} \gets \frac{w_k}{1+\eta_t\mu\beta_k} - \gamma_i\frac{P(\eta_t N(K-1)\gamma_i^{-1}e^{x_i^\top(\frac{\tilde{w}_k}{1+\eta_t\mu\beta_k} - \frac{\tilde{w}_{y_i}}{1+\eta_t\mu\beta_{y_i}})-u_i})}{1+\eta_t\mu\beta_k} x_i$
    \STATE $w_{y_i} \gets \frac{w_{y_i}}{1+\eta_t\mu\beta_{y_i}} + \gamma_i\frac{P(\eta_t N(K-1)\gamma_i^{-1}e^{x_i^\top(\frac{\tilde{w}_k}{1+\eta_t\mu\beta_k} - \frac{\tilde{w}_{y_i}}{1+\eta_t\mu\beta_{y_i}})-u_i})}{1+\eta_t\mu\beta_{y_i}} x_i$
   \ENDFOR
\end{algorithmic}
\end{algorithm}

\subsection{Bound on step size}\label{app:Implicit_SGD_step_size_bound}
Here we will prove that the step size magnitude of Implicit SGD with a single datapoint and sampled class with respect to $w$ is bounded as ${O(x_i^\top(\frac{\tilde{w}_k}{1+\eta\mu\beta_k} - \frac{\tilde{w}_{y_i}}{1+\eta\mu\beta_{y_i}})-\tilde{u}_i)}$. We will do so by considering the two cases $u_i' > \tilde{u}_i$ and $u_i' < \tilde{u}_i$ separately, where $u_i'$ denotes the optimal value of $u_i$ in the Implicit SGD update and $\tilde{u}_i$ is its value at the previous iterate.\\

\noindent\textbf{Case: $u_i' > \tilde{u}_i$}\\
Let $a'$ denote the optimal value of $a$ in the Implicit SGD update. From~(\ref{eq:a_u_function})
\begin{align*}
a' &= P\left(\eta N(K-1)\gamma_i^{-1}e^{x_i^\top(\frac{\tilde{w}_k}{1+\eta\mu\beta_k} - \frac{\tilde{w}_{y_i}}{1+\eta\mu\beta_{y_i}})-u_i'}\right)\\
&\leq  P\left(\eta N(K-1)\gamma_i^{-1}e^{x_i^\top(\frac{\tilde{w}_k}{1+\eta\mu\beta_k} - \frac{\tilde{w}_{y_i}}{1+\eta\mu\beta_{y_i}})-\tilde{u}_i}\right)
\end{align*}
where $u_i'$ is replace by $\tilde{u}_i$ and we have used the monotonicity of the Lambert-W function $P$.
Now using the fact that $P(z) = O(\log (z))$,
\begin{align*}
a' &= O(x_i^\top(\frac{\tilde{w}_k}{1+\eta\mu\beta_k} - \frac{\tilde{w}_{y_i}}{1+\eta\mu\beta_{y_i}})-\tilde{u}_i + \log(\eta N (K-1)\gamma_i^{-1}))\\
&= O(x_i^\top(\frac{\tilde{w}_k}{1+\eta\mu\beta_k} - \frac{\tilde{w}_{y_i}}{1+\eta\mu\beta_{y_i}})-\tilde{u}_i )
\end{align*}

\noindent\textbf{Case: $u_i' < \tilde{u}_i$}\\
If $u_i' < \tilde{u}_i$ then we can lower bound $a'$ from~(\ref{eq:a_dash_bound}) as $a' \leq  \eta N\gamma_i^{-1}$.\\

\noindent\textbf{Combining cases}\\
Putting together the two cases, 
\begin{align*}
a' &= O(\max\{x_i^\top(\frac{\tilde{w}_k}{1+\eta\mu\beta_k} - \frac{\tilde{w}_{y_i}}{1+\eta\mu\beta_{y_i}})-\tilde{u}_i, \,\eta N\gamma_i^{-1}\}) \\
&= O(x_i^\top(\frac{\tilde{w}_k}{1+\eta\mu\beta_k} - \frac{\tilde{w}_{y_i}}{1+\eta\mu\beta_{y_i}})-\tilde{u}_i).
\end{align*}
From (\ref{eq:implicit_single_class_w_update}) we know that the step size magnitude is proportional to $a'$. Thus the step size magnitude is also $ O(x_i^\top(\frac{\tilde{w}_k}{1+\eta\mu\beta_k} - \frac{\tilde{w}_{y_i}}{1+\eta\mu\beta_{y_i}})-\tilde{u}_i)$.

\subsection{Single datapoint, multiple classes.}\label{app:Implicit_SGD_single_multiple}
Consider the case where only one datapoint $i$, but multiple classes $\{k_j:k_y\neq y_i\}_{j=1}^m$ are sampled each iteration. Like in Appendix~\ref{app:Implicit_SGD_simpled}, we will be able to reduce the implicit update to a one-dimensional strongly-convex optimization problem. The resulting problem may be solved using any standard convex optimization method, such as Newton's method. We do not derive upper and lower bounds for a bisection method as we did in Appendix~\ref{app:Implicit_SGD_simpled}.

Let us rewrite the double-sum formulation from (\ref{eq:f}) as $f(u,  W) = \mathbb{E}_{i,C_i}[f_{i,C_i}(u,  W)]$ where $i$ is a uniformly sampled datapoint, $C_i$ is a set of $m$ uniformly sampled classes from $\{1,...,K\}-\{y_i\}$ (without replacement) and 
\begin{equation*}
f_{i,C_i}(u,  W) = N( u_i +
  e^{-u_i}+\alpha\sum_{k\in C_i}e^{x_i^\top (w_k - w_{y_i})-u_i}) 
   +\frac{\mu}{2}\sum_{k\in C_i\cup\{y_i\}}\beta_k \|w_k\|_2^2,
\end{equation*}
where $\alpha^{-1} = P(k\in C_i| k\neq y_i) = 1-\prod_{j=1}^m(1-\frac{1}{K-j})$,
\begin{align*}
\beta_k^{-1} &= P(k\in C_i\cup\{y_i\})\\
&= P(k= y_i) + P(k\in C_i| k\neq y_i)P(k\neq y_i)\\
&= \frac{n_k + \alpha^{-1}(N-n_k)}{N}
\end{align*}
and $n_k= |\{i:y_i=k, i=1,...,N\}|$. Using the same derivation as in Appendix~\ref{app:Implicit_SGD_simpled}, the implicit SGD update is
\begin{align}
\min_{u_i,~\{w_k\}_{k\in C_i\cup\{y_i\}}} &2\eta\left(N( u_i +
  e^{-u_i}+\alpha\sum_{k\in C_i}e^{x_i^\top (w_k - w_{y_i})-u_i}) 
   +\frac{\mu}{2}\sum_{k\in C_i\cup\{y_i\}}\beta_k \|w_k\|_2^2\right) \nonumber\\
   &+ (u_i - \tilde{u}_i)^2 +  \sum_{k\in C_i\cup\{y_i\}}\|w_k-\tilde{w}_k\|_2^2.
\end{align}
The goal is to simplify this multivariate minimization problem into a one-dimensional strongly convex minimization problem.
The first trick we will use is to reparameterize $u_i = v_i -x_i^\top w_{y_i}$ for some $v_i\in\mathbb{R}$. This changes the $e^{x_i^\top (w_k - w_{y_i})-u_i}$ factors to $e^{x_i^\top w_k -v_i}$, decoupling $w_k$ and $w_{y_i}$, which will make the optimization easier. The problem becomes:
\begin{align}
\min_{v_i,~\{w_k\}_{k\in C_i\cup\{y_i\}}} &2\eta\left(N( v_i -x_i^\top w_{y_i} +
  e^{x_i^\top w_{y_i}-v_i}+\alpha\sum_{k\in C_i}e^{x_i^\top w_k -v_i}) 
   +\frac{\mu}{2}\sum_{k\in C_i\cup\{y_i\}}\beta_k \|w_k\|_2^2\right) \nonumber\\
   &+ (v_i -x_i^\top w_{y_i} - \tilde{u}_i)^2 +  \sum_{k\in C_i\cup\{y_i\}}\|w_k-\tilde{w}_k\|_2^2.\label{eq:v_first_min}
\end{align}
Since $v_i = u_i+x_i^\top w_{y_i}$ is a linear transformation, (\ref{eq:v_first_min}) is jointly strongly convex in $v_i$ and $\{w_k\}_{k\in C_i\cup\{y_i\}}$. 
Bringing the $w_k$ minimizations inside yields
\begin{align}
\min_{v_i}~ &2\eta Nv_i \nonumber\\
& +\min_{w_{y_i}}\left\{-2\eta Nx_i^\top w_{y_i} +
  2\eta Ne^{x_i^\top w_{y_i}-v_i} + \eta\mu\beta_{y_i} \|w_{y_i}\|_2^2 + (v_i -x_i^\top w_{y_i} - \tilde{u}_i)^2 +  \|w_{y_i}-\tilde{w}_{y_i}\|_2^2 \right\}\nonumber\\
  &+\sum_{k\in C_i} \min_{w_k}\left\{2\eta N\alpha e^{x_i^\top w_k -v_i}
   +\eta\mu\beta_k \|w_k\|_2^2  +  \|w_k-\tilde{w}_k\|_2^2\right\}\label{eq:still_w_k}.
\end{align}
In Appendix~\ref{app:Implicit_SGD_simpled} we were able to reduce the dimensionality of the problem by introducing an auxiliary variable $b$ to separate the exponential terms from the norm terms. We will do a similar thing here. Let us first focus on the inner minimization for $k\in C_i$.
\begin{align*}
&\min_{w_k}\left\{2\eta N\alpha e^{x_i^\top w_k -v_i} +\eta\mu\beta_k \|w_k\|_2^2  +  \|w_k-\tilde{w}_k\|_2^2\right\}\\
&=\min_{b}\left\{2\eta N\alpha e^{b -v_i}  + \min_{w_k}\left\{\eta\mu\beta_k \|w_k\|_2^2  +  \|w_k-\tilde{w}_k\|_2^2:~ b = x_i^\top w_k \right\} \right\}\\
&=\min_{b}\left\{2\eta N\alpha e^{b -v_i}  + \max_{\lambda\in\mathbb{R}}\min_{w_k}\left\{\eta\mu\beta_k \|w_k\|_2^2  +  \|w_k-\tilde{w}_k\|_2^2 + 2\lambda(b- x_i^\top w_k) \right\} \right\}
\end{align*}
where we have taken the Lagrangian in the final line. The solution for $w_k$ in terms of $\lambda$ is 
\begin{align*}
w_k = \frac{\tilde{w}_k}{1+\eta\mu\beta_k} + \frac{\lambda}{1+\eta\mu\beta_k} x_i.
\end{align*}
Thus we know that our optimal $w_k$ must satisfy $w_k = \frac{\tilde{w}_k}{1+\eta\mu\beta_k} - a_k \frac{x_i}{\|x_i\|_2^2}$ for some $a_k\in\mathbb{R}$. It can similarly be shown that $w_{y_i} = \frac{\tilde{w}_{y_i}}{1+\eta\mu\beta_{y_i}} + a_{y_i} \frac{x_i}{\|x_i\|_2^2}$ for some $a_{y_i}\in\mathbb{R}$. Substituting this into (\ref{eq:still_w_k}) and dropping constant terms yields
\begin{align}
\min_{v_i}~ &2v_i\left(\eta N-\frac{x_i^\top\tilde{w}_{y_i}}{1+\eta\mu\beta_{y_i}} -\tilde{u}_i\right) + v_i^2\nonumber\\
& +\min_{a_{y_i}}\left\{
  2e^{a_{y_i}}\left(\eta Ne^{\frac{x_i^\top\tilde{w}_{y_i}}{1+\eta\mu\beta_{y_i}} -v_i}\right) +2a_{y_i} \left(-v_i -\eta N + \frac{x_i^\top\tilde{w}_{y_i}}{1+\eta\mu\beta_{y_i}} + \tilde{u}_i\right)  +  a_{y_i}^2(1+\|x_i\|_2^{-2}(1+\eta\mu\beta_{y_i}))  \right\}\nonumber\\
  &+\sum_{k\in C_i} \min_{a_k}\left\{2e^{-a_k}\left(\eta N\alpha e^{ \frac{x_i^\top\tilde{w}_k}{1+\eta\mu\beta_k} -v_i}\right) + a_k^2(\|x_i\|^{-2}(1+\eta\mu\beta_k))\right\}.\label{eq:max_a_ks}
\end{align}
Using the same techniques as in Appendix \ref{app:Implicit_SGD_simpled} we can analytically solve for the $a$ values:
\begin{align*}
a_{y_i}(v_i) &= \frac{\eta N \|x_i\|_2^2-x_i^\top\tilde{w}_{y_i}\|x_i\|_2^2/(1+\eta\mu\beta_{y_i}) + (v_i-\tilde{u}_i)\|x_i\|_2^2}{1+\eta\mu\beta_{y_i} + \|x_i\|_2^{2}} - P(\sigma(v_i))\\
a_k(v_i) &= P\left(\frac{\eta}{1+\eta\mu\beta_k} \|x_i\|_2^{2} N\alpha e^{ \frac{x_i^\top\tilde{w}_k}{1+\eta\mu\beta_k} -v_i}\right).
\end{align*}
where $\sigma(v_i) = \frac{\eta N \|x_i\|_2^2}{1+\eta\mu\beta_{y_i} + \|x_i\|_2^{2}}\exp\left(\frac{x_i^\top\tilde{w}_{y_i} - v_i(1+\eta\mu\beta_{y_i})+(\eta N -\tilde{u}_i)\|x_i\|_2^2}{1+\eta\mu\beta_{y_i} + \|x_i\|_2^{2}}\right)$. 
Substituting these values into (\ref{eq:max_a_ks}) yields
\begin{align}
\min_{v_i}~ &2v_i\left(1+\eta N-\frac{x_i^\top\tilde{w}_{y_i}}{1+\eta\mu\beta_{y_i}} - \tilde{u}_i)\right) + v_i^2\nonumber\\
& -2a_{y_i}(v_i)\left(v_i +\eta N - \frac{x_i^\top\tilde{w}_{y_i}}{1+\eta\mu\beta_{y_i}} - \tilde{u}_i +1 + \|x_i\|_2^{-2}(1+\eta\mu \beta_{y_i}) \right) + a_{y_i}(v_i)^2(1 + \|x_i\|_2^{-2}(1+\eta\mu \beta_{y_i}))\nonumber\\
  &+\sum_{k\in C_i} a_k(v_i)(1+a_k(v_i))\|x_i\|^{-2}(1+\eta\mu \beta_k).\label{eq:final_v_i_problem}
\end{align}
This is a one-dimensional strongly convex minimization problem in $v_i$. The optimal $v_i$ can be solved for using any standard convex optimization method, such as Newton's method. Each iteration in such a method will take $O(m)$ since it is necessary to calculate $a_k(v_i)$, $\partial_{v_i} a_k(v_i)$ and $\partial_{v_i}^2 a_k(v_i)$ for all $k\in C_i\cup\{y_i\}$. The first derivatives are easily calculated,
\begin{align*}
\partial_{v_i} a_{y_i}(v_i) &= \frac{\|x_i\|_2^2}{1+\eta\mu\beta_{y_i} + \|x_i\|_2^{2}}  +\frac{1+\eta\mu\beta_{y_i}}{1+\eta\mu\beta_{y_i} + \|x_i\|_2^{2}} \frac{P(\sigma(v_i))}{1+P(\sigma(v_i))} \\
\partial_{v_i} a_k(v_i) &= -\frac{a_k(v_i)}{1+a_k(v_i)},
\end{align*}
as are the second derivatives,
\begin{align*}
\partial_{v_i}^2 a_{y_i}(v_i) &= -\left(\frac{1+\eta\mu\beta_{y_i}}{1+\eta\mu\beta_{y_i} + \|x_i\|_2^{2}}\right)^2 \frac{P(\sigma(v_i))}{(1+P(\sigma(v_i)))^3} \\
\partial_{v_i}^2 a_k(v_i) &= \frac{a_k(v_i)^2}{(1+a_k(v_i))^3}.
\end{align*}

\subsection{Multiple datapoints, multiple classes}\label{app:Implicit_SGD_multiple}
Consider the case where $n$ datapoints and $m$ classes are sampled each iteration. Using similar methods to Appendix \ref{app:Implicit_SGD_single_multiple}, we will reduce the implicit update to an $n$ dimensional strongly convex optimization problem.

Let us rewrite the double-sum formulation from (\ref{eq:f}) as $f(u,  W) = \mathbb{E}_{I,C}[f_{I,C}(u,  W)]$ where $I$ is a set of $n$ datapoints uniformly sampled from $1,...,N$ (without replacement), $C$ is a set of $m$ uniformly sampled classes from $1,...,K$ (without replacement). The sampled function is of the form
\begin{align*}
f_{I,C}(u,  W) =& \sum_{i\in I}\left( \alpha_n(u_i + e^{-u_i})+\alpha_m \sum_{k\in C}I[k\neq y_i]e^{x_i^\top (w_k - w_{y_i})-u_i}\right) +\frac{\mu}{2}\sum_{k\in C\cup_{i\in I}\{y_i\}}\beta_k \|w_k\|_2^2,
\end{align*}
where
\begin{align*}
\alpha_n &= P(i\in I)^{-1} = \left(1-\prod_{j=0}^{n-1}(1-\frac{1}{N-j})\right)^{-1}\\
\alpha_m/\alpha_n &= P(k\in C)^{-1} =\left(1-\prod_{j=0}^{m-1}(1-\frac{1}{K-j})\right)^{-1}\\
\beta_k &= P(k\in C \cup_{i\in I}\{y_i\})^{-1} = \left(P(k\in C) + P(k\in \cup_{i\in I}\{y_i\}) - P(k\in C)P(k\in \cup_{i\in I}\{y_i\}) \right)^{-1}\\
&~~~~~P(k\in \cup_{i\in I}\{y_i\}) = 1 - \prod_{j=0}^{n-1}\left(1-\frac{|\{i:y_i=k\}|}{N-j}\right).
\end{align*}
It will be useful to group the classes that appear in $\cup_{i\in I}\{y_i\}$ and those that only appear in $C$:
\begin{align*}
f_{I,C}(u,  W) =& \sum_{k\in \cup_{i\in I}\{y_i\}}\sum_{i\in I}\left(I[k=y_i]\alpha_n(u_i + e^{-u_i}) + I[k\neq y_i, k\in C]\alpha_m e^{x_i^\top (w_k - w_{y_i})-u_i} \right) + \frac{\mu}{2}\beta_k \|w_k\|_2^2\\
&+ \sum_{k\in C-\cup_{i\in I}\{y_i\}}\sum_{i\in I}\alpha_m e^{x_i^\top (w_k - w_{y_i})-u_i} + \frac{\mu}{2}\beta_k \|w_k\|_2^2.
\end{align*}
The implicit SGD update is
\begin{align*}
\min_{\substack{\{u_i\}_{i\in I} \\ \{w_k\}_{k\in C\cup_{i\in I}\{y_i\}}}} &2\eta\bigg(\sum_{k\in \cup_{i\in I}\{y_i\}}\sum_{i\in I}\left(I[k=y_i]\alpha_n(u_i + e^{-u_i}) + I[k\neq y_i, k\in C]\alpha_m e^{x_i^\top (w_k - w_{y_i})-u_i} \right) + \frac{\mu}{2}\beta_k \|w_k\|_2^2\\
&+ \sum_{k\in C-\cup_{i\in I}\{y_i\}}\sum_{i\in I}\alpha_m e^{x_i^\top (w_k - w_{y_i})-u_i} + \frac{\mu}{2}\beta_k \|w_k\|_2^2\bigg) \\
   &+ \sum_{i\in I}(u_i - \tilde{u}_i)^2 +  \sum_{k\in C\cup_{i\in I}\{y_i\}}\|w_k-\tilde{w}_k\|_2^2.
\end{align*}
Like in Appendix~\ref{app:Implicit_SGD_single_multiple}, the first step to simplifying this equation is to reparameterize $u_i = v_i -x_i^\top w_{y_i}$ for some $v_i\in\mathbb{R}$ and to bring the $w_k$ minimizations inside:
\begin{align}
\min_{\{v_i\}_{i\in I} } &2\eta \alpha_n v_i + \sum_{k\in \cup_{i\in I}\{y_i\}}\min_{w_k}\bigg\{\sum_{i\in I}\bigg(I[k=y_i]\left(2\eta \alpha_n(-x_i^\top w_k + e^{x_i^\top w_k-v_i}) + (v_i -x_i^\top w_k - \tilde{u}_i)^2\right) \nonumber\\
&~~~~~~~~~~~~~~~~~~~~~~~~~~~~~~~~~~~~~~~~~~~~+ I[k\neq y_i, k\in C]2\eta \alpha_m e^{x_i^\top w_k -v_i} \bigg) + \eta\mu\beta_k \|w_k\|_2^2 + \|w_k-\tilde{w}_k\|_2^2\bigg\} \nonumber \\
&+ \sum_{k\in C-\cup_{i\in I}\{y_i\}} \min_{w_k}\bigg\{\sum_{i\in I}2\eta \alpha_m e^{x_i^\top w_k -v_i} + \eta\mu \beta_k \|w_k\|_2^2 + \|w_k-\tilde{w}_k\|_2^2 \bigg\}.\label{eq:mult_mult_v_original}
\end{align}
As done in Appendix~\ref{app:Implicit_SGD_single_multiple}, the inner minimizations can be solved analytically by introducing constrained auxiliary variables $b_{ki} = x_i^\top w_k$ and optimizing the dual. We'll do this separately for $k\in \cup_{i\in I}\{y_i\}$ and $k\in C-\cup_{i\in I}\{y_i\}$.\\

\noindent\textbf{For datapoint labels $k\in \cup_{i\in I}\{y_i\}$},
\begin{align*}
&\min_{w_k}\sum_{i\in I}\bigg(I[k=y_i]\left(2\eta \alpha_n(-x_i^\top w_k + e^{x_i^\top w_k-v_i}) + (v_i -x_i^\top w_k - \tilde{u}_i)^2\right) \\
&~~~~~~~~~~+ I[k\neq y_i, k\in C]2\eta \alpha_m e^{x_i^\top w_k -v_i} \bigg) + \eta\mu\beta_k \|w_k\|_2^2 + \|w_k-\tilde{w}_k\|_2^2\\
&=\min_{b_{ki}}\sum_{i\in I}\bigg(I[k=y_i]\left(2\eta \alpha_n(-b_{ki} + e^{b_{ki}-v_i}) + (v_i -b_{ki} - \tilde{u}_i)^2\right) + I[k\neq y_i, k\in C]2\eta \alpha_m e^{b_{ki} -v_i} \bigg) \\
&~~~~~~~~~~~~+ \min_{w_k}\left\{\eta\mu\beta_k \|w_k\|_2^2 + \|w_k-\tilde{w}_k\|_2^2:~b_{ki} = x_i^\top w_k\right\}.
\end{align*}
Focusing on the minimization over $w_k$:
\begin{align*}
&\min_{w_k}\left\{\eta\mu\beta_k \|w_k\|_2^2 + \|w_k-\tilde{w}_k\|_2^2:~b_{ki} = x_i^\top w_k\right\}\\
&= \max_{\lambda_{ki}} \min_{w_k} \eta\mu\beta_k \|w_k\|_2^2 + \|w_k-\tilde{w}_k\|_2^2 +2\sum_{i\in I}\lambda_{ki}(b_{ki} - x_i^\top w_k).
\end{align*}
The solution for $w_k$ in terms of $\lambda_{ki}$ is
\begin{align*}
w_k = \frac{\tilde{w}_k + \sum_{i\in I}\lambda_{ki}x_i}{1+\eta\mu\beta_k}
\end{align*}
Dropping constant terms, the dual becomes
\begin{align*}
&\max_{\lambda_{ki}} -\frac{\|\sum_{i\in I}\lambda_{ki}x_i \|_2^2}{1+\eta\mu\beta_k } + 2\sum_{i\in I}\lambda_{ki}(b_{ki} - \frac{x_i^\top \tilde{w}_k}{1+\eta\mu\beta_k})\\
&= \max_{\lambda_{ki}} - \lambda_k^\top Q_k \lambda_k +2\lambda_k^\top \left(b_k - \frac{X_I^\top \tilde{w}_k}{1+\eta\mu\beta_k}\right)\\
&= \left(b_k - \frac{X_I^\top \tilde{w}_k}{1+\eta\mu\beta_k}\right)^\top Q_k^{-1} \left(b_k - \frac{X_I^\top \tilde{w}_k}{1+\eta\mu\beta_k}\right)\\
&= \left\|b_k - \frac{X_I^\top \tilde{w}_k}{1+\eta\mu\beta_k}\right\|_{Q_k^{-1}}^2,
\end{align*}
where $Q_{k,ij} = \frac{x_i^\top x_j}{1+\eta\mu\beta_k}$ and $X_I = (x_i)_{i\in I}\in\mathbb{R}^{D\times n}$ and the optimal $\lambda = Q_k^{-1} \left(b_k - \frac{X_I^\top \tilde{w}_k}{1+\eta\mu\beta_k}\right)$. Now we can solve for $b_k$,
\begin{align*}
&\min_{b_{ki}}\sum_{i\in I} I[k=y_i]\left(2\eta \alpha_n(-b_{ki} + e^{b_{ki}-v_i}) + (v_i -b_{ki} - \tilde{u}_i)^2\right) + I[k\neq y_i, k\in C]2\eta \alpha_m e^{b_{ki} -v_i} + \left\|b_k - \frac{X_I^\top \tilde{w}_k}{1+\eta\mu\beta_k}\right\|_{Q_k^{-1}}^2.
\end{align*}
Setting to zero the derivative with respect to $b_k\in\mathbb{R}^n$ and dividing by 2:
\begin{align}
0&= I[k=y_I]\circ \left(\eta \alpha_n(-\mathbf{1} + e^{b_k-v_I}) + b_k + \tilde{u}_I - v_I \right) + I[k\neq y_I, k\in C] \circ \eta \alpha_m e^{b_k -v_I} + Q_k^{-1} \left(b_k - \frac{X_I^\top \tilde{w}_k}{1+\eta\mu\beta_k}\right)\nonumber\\
&= diag(a) e^{b_k} + A_k b_k - h_k \label{eq:diag_original}
\end{align}
where $\circ$ denotes the element-wise product, $diag(a)$ is a diagonal matrix, $\mathbf{1}$ denotes the vectors of all ones, $v_I = (v_i)_{i\in I}\in\mathbb{R}^{n}$, likewise for $\tilde{u}_I$ and $y_I$, and 
\begin{align*}
a_k &= I[k=y_I]\circ \eta \alpha_n e^{-v_I}  + I[k\neq y_I, k\in C] \circ \eta \alpha_m e^{ -v_I}\\
A_k &= diag(I[k=y_I]) + Q_k^{-1}\\
h_k &= I[k=y_I]\circ \left(\eta \alpha_n\mathbf{1} - \tilde{u}_I + v_I \right) + Q_k^{-1}  \frac{X_I^\top \tilde{w}_k}{1+\eta\mu\beta_k}.
\end{align*}
Multiplying (\ref{eq:diag_original}) on the left by $A_k^{-1}$, letting $z_k = A_k^{-1}h_k-b_k$ and multiplying on the right by $diag(e^{z_k})$ yields
\begin{align*}
z_k \circ e^{z_k} &= A_k^{-1} ( a \circ e^{A_k^{-1}h_k} ).
\end{align*}
The solution for $z_k$ decomposes into separate Lambert-W functions:
\begin{align*}
z_{k} &= P(A_k^{-1} ( a \circ e^{A_k^{-1}h_k} ))
\end{align*}
where $P$ is the principle branch of the Lambert-W function applied component-wise. The solution for $b_k$ is thus 
\begin{equation}\label{eq:b_k}
b_k(v_I) = A_k^{-1}h_k-P(A_k^{-1} ( a \circ e^{A_k^{-1}h_k} ))
\end{equation}
where $b_k$ is a function of the variable $v_I$, which is the only unknown variable that we are yet to minimize over.\\

\noindent\textbf{For pure class labels $k\in C-\cup_{i\in I}\{y_i\}$} the procedure is nearly identical for the the datapoint labels. The optimal value of $w_k$ is 
\begin{align*}
w_k = \frac{\tilde{w}_k + \sum_{i\in I}\lambda_{ki}x_i}{1+\eta\mu\beta_k}
\end{align*}
where $\lambda = Q_k^{-1} \left(b_k - \frac{X_I^\top \tilde{w}_k}{1+\eta\mu\beta_k}\right)$ and
\begin{align*}
b_k(v_I) &= \frac{X_I^\top \tilde{w}_k}{1+\eta\mu\beta_k}-P\left(\eta \alpha_n Q_k e^{\frac{X_I^\top \tilde{w}_k}{1+\eta\mu\beta_k}-v_I}\right).
\end{align*}

\noindent\textbf{Final optimization problem}\\
Substituting the optimal values of $b_k$ in (\ref{eq:mult_mult_v_original}) yields the final optimization problem
\begin{align*}
&\min_{\{v_i\}_{i\in I} } 2\eta \alpha_n v_i \\
&+ \sum_{k\in \cup_{i\in I}\{y_i\}}\sum_{i\in I}I[k=y_i]\left(2\eta \alpha_n(-b_{ki}(v_I) + e^{b_{ki}(v_I)-v_i}) + (v_i -b_{ki}(v_I) - \tilde{u}_i)^2\right) + I[k\neq y_i, k\in C]2\eta \alpha_m e^{b_{ki}(v_I) -v_i}  \\
&~~~~~~~~~~~~+  \left\|b_k(v_I) - \frac{X_I^\top \tilde{w}_k}{1+\eta\mu\beta_k}\right\|_{Q_k^{-1} }^2  \\
&+ \sum_{k\in C-\cup_{i\in I}\{y_i\}} \sum_{i\in I}2\eta \alpha_m e^{\frac{X_I^\top \tilde{w}_k}{1+\eta\mu\beta_k}-P\left(\eta \alpha_n Q_k e^{\frac{X_I^\top \tilde{w}_k}{1+\eta\mu\beta_k}-v_I}\right) -v_i} + \left\|P\left(\eta \alpha_n Q_k e^{\frac{X_I^\top \tilde{w}_k}{1+\eta\mu\beta_k}-v_I}\right)\right\|_{Q_k^{-1} }^2.
\end{align*}
where $b_k(v_I)$ is from (\ref{eq:b_k}).
This is a strongly convex optimization problem in $v_I\in\mathbb{R}^n$. Using standard first order gradient methods, it can be solved to $\epsilon>0$ accuracy in $O(\log(\epsilon^{-1}))$ iterations. The cost per iteration is $O(n^2(n+m))$ for the matrix multiplications and $O((n+m)n^3)$ for the matrix inversions. Note that the matrix inversions do not depend on $v_I$ and so they only have to be performed once. Furthermore, if the same minibatches are used each epoch, then the inverted matrices can be calculated just once and stored. The amortized matrix inversion cost is therefore expected to be dominated by the $O(n^2(n+m)\log(\epsilon^{-1}))$ cost for solving for $v_I$ and the $O(nmD)$ cost of taking the $x_i^\top \tilde{w}_k$ inner products each iteration.

Note that we have assumed that $Q_k$ is invertible. As long as the vectors $\{x_i\}_{i\in I}$ are independent, this will be the case. If not, then a similar method as above can be developed where a basis of $\{x_i\}_{i\in I}$ is used.

\section{U-max pseudocode}\label{app:umax_alg}
~\vspace{-0.75cm}
\begin{algorithm}[h]
   \caption{U-max for a single datapoint and multiple classes sampled per iteration.}
   \label{alg:umax}
\begin{algorithmic}
      \STATE {\bfseries Input:} Data $\mathcal{D} = \{(y_i, x_i)\}_{i=1}^N$, number of classes to sample each iteration $m$, number of iterations~$T$, learning rate $\eta_t$, threshold $\delta>0$, constants $\alpha$ and $\beta$, initial $u,W$.
   \STATE {\bfseries Ouput:} $W$
   \vspace{0.5cm}

   \FOR{$t=1$ {\bfseries to} $T$}
   \STATE  \texttt{Sample datapoint and classes}
   \STATE $i \sim unif(\{1,...,N\})$
   \STATE $k_j \sim unif(\{1,...,K\}-\{y_i\})$ for $j=1,...,m$ (with replacement)
   
   \vspace{0.5cm}
   \STATE   \texttt{Increase $u_i$}
   \IF{$u_i < \log(1+\sum_{j=1}^m e^{x_i^\top (w_{k_j} - w_{y_i})}) - \delta$}
   \STATE $u_i \gets \log(1+\sum_{j=1}^m e^{x_i^\top (w_{k_j} - w_{y_i})})$
   \ENDIF
   
   \vspace{0.5cm}
   \STATE  \texttt{SGD step}
    \STATE $w_{k_j} \gets w_{k_j} - \eta_t N(K-1)/m\cdot   e^{x_i^\top (w_{k_j} - w_{y_i})-u_i}x_i - \eta_t\mu \beta_{k_j} w_{k_j}$ for $j=1,...,m$
    \STATE $w_{y_i} \gets w_{y_i} + \eta_t N(K-1)/m\cdot  \sum_{j=1}^me^{x_i^\top (w_{k_j} - w_{y_i})-u_i}x_i - \eta_t\mu \beta_{y_i} w_{y_i}$
    \STATE $u_i \gets u_i - \eta_t N(1-e^{-u_i} - (K-1)/m\cdot  \sum_{j=1}^me^{x_i^\top (w_{k_j} - w_{y_i})-u_i})$
   \ENDFOR
\end{algorithmic}
\end{algorithm}

\section{Proof of convergence of U-max method}\label{app:proof_Umax}
In this section we will prove the claim made in Proposition~\ref{thm:Umax}, that U-max converges to the softmax optimum. Before proving the proposition, we will need a lemma.

\begin{lemma}\label{lemma:f_decrease} For any $\delta >0$, if $u_i\leq \log(1+e^{x_i^\top (w_k - w_{y_i})}) - \delta$ then setting $u_i= \log(1+e^{x_i^\top (w_k - w_{y_i})})$ decreases $f(u,W)$ by at least $\delta^2/2$.
\end{lemma}

\begin{proof}
As in Lemma~\ref{lemma:strong_cvx}, let $\theta = (u^\top,w_1^\top,...,w_k^\top)\in\mathbb{R}^{N+KD}$. Then setting $u_i= \log(1+e^{x_i^\top (w_k - w_{y_i})})$ is equivalent to setting $\theta = \theta+\Delta e_i$ where $e_i$ is the $i^{th}$ canonical basis vector and $\Delta = \log(1+e^{x_i^\top (w_k - w_{y_i})}) - u_i \geq \delta$. By a second order Taylor series expansion
\begin{align}\label{eq:Taylor}
f(\theta) - f(\theta+\Delta e_i) &\geq  \nabla f(\theta+\Delta e_i)^\top e_i\Delta +\frac{\Delta^2}{2}e_i^\top\nabla^2 f(\theta+\lambda\Delta e_i) e_i
\end{align}
for some $\lambda\in[0,1]$. Since the optimal value of $u_i$ for a given value of $W$ is $u_i^\ast(W) = \log(1+\sum_{k\neq y_i}e^{x_i^\top (w_k - w_{y_i})}) \geq \log(1+e^{x_i^\top (w_k - w_{y_i})})$, we must have $\nabla f(\theta+\Delta e_i)^\top e_i \leq 0$. From Lemma~\ref{lemma:strong_cvx} we also know that 
\begin{align*}
e_i^\top\nabla^2 f(\theta+\lambda\Delta e_i) e_i &= \exp(-(u_i+\lambda\Delta))+\sum_{k\neq y_i}e^{x_i^\top (w_k - w_{y_i})-(u_i+\lambda\Delta)}\\
&= \exp(-\lambda\Delta)e^{-u_i}(1+\sum_{k\neq y_i}e^{x_i^\top (w_k - w_{y_i})})\\
&= \exp(-\lambda\Delta)\exp(-(\log(1+e^{x_i^\top (w_k - w_{y_i})}) - \Delta))(1+\sum_{k\neq y_i}e^{x_i^\top (w_k - w_{y_i})})\\
&\geq \exp(\Delta-\lambda\Delta)\\
&\geq \exp(\Delta-\Delta)\\
&= 1.
\end{align*}
Putting in bounds for the gradient and Hessian terms in~(\ref{eq:Taylor}),
\begin{align*}
f(\theta) - f(\theta+\Delta e_i) \geq  \frac{\Delta^2}{2} \geq  \frac{\delta^2}{2}.
\end{align*}
\end{proof}

Now we are in a position to prove Proposition~\ref{thm:Umax}.

\begin{proof}[Proof of Proposition~\ref{thm:Umax}]
Let $\theta^{(t)} = (u^{(t)}, W^{(t)})\in\Theta$ denote the value of the $t^{th}$ iterate. Here $\Theta=\{\theta: \, \|W\|_2^2\leq B_W^2, u_i \leq B_u\}$ is a convex set containing the optimal value of $f(\theta)$.

Let $\pi_i^{(\delta)}(\theta)$ denote the operation of setting $u_i= \log(1+e^{x_i^\top (w_k - w_{y_i})})$ if $u_i\leq \log(1+e^{x_i^\top (w_k - w_{y_i})}) - \delta$. If indices $i,k$ are sampled for the stochastic gradient and $u_i\leq \log(1+e^{x_i^\top (w_k - w_{y_i})}) - \delta$, then the value of $f$ at the $t+1^{st}$ iterate is bounded as
\begin{align*}
f(\theta^{(t+1)}) &= f(\pi_i(\theta^{(t)}) - \eta_t \nabla f_{ik}(\pi_i(\theta^{(t)})))\\
& \leq f(\pi_i(\theta^{(t)})) +\max_{\theta\in\Theta} \|\eta_t \nabla f_{ik}(\pi_i(\theta))\|_2\max_{\theta\in\Theta} \|\nabla f (\theta)\|_2\\
& \leq f(\pi_i(\theta^{(t)})) + \eta_t B_f^2\\
& \leq f(\theta^{(t)}) -\delta^2/2 + \eta_t B_f^2\\
& \leq f(\theta^{(t)}- \eta_t \nabla f_{ik}(\theta^{(t)})) -\delta^2/2 + 2\eta_t B_f^2\\
& \leq f(\theta^{(t)}- \eta_t \nabla f_{ik}(\theta^{(t)})),
\end{align*}
since $\eta_t\leq \delta^2/(4 B_f^2)$ by assumption.
Alternatively if $u_i\geq \log(1+e^{x_i^\top (w_k - w_{y_i})}) - \delta$ then
\begin{align*}
f(\theta^{(t+1)}) &= f(\pi_i(\theta^{(t)}) - \eta_t \nabla f_{ik}(\pi_i(\theta^{(t)})))\\
& = f(\theta^{(t)}- \eta_t \nabla f_{ik}(\theta^{(t)})).
\end{align*}
Either way $f(\theta^{(t+1)}) \leq f(\theta^{(t)}- \eta_t \nabla f_{ik}(\theta^{(t)}))$. Taking expectations with respect to $i,k$, 
\begin{align*}
\mathbb{E}_{ik}[f(\theta^{(t+1)})] & \leq \mathbb{E}_{ik}[f(\theta^{(t)}- \eta_t \nabla f_{ik}(\theta^{(t)}))].
\end{align*}
Finally let $P$ denote the projection of $\theta$ onto $\Theta$. Since $\Theta$ is a convex set containing the optimum we have $f(P(\theta)) \leq f(\theta)$ for any $\theta$, and so
\begin{align*}
\mathbb{E}_{ik}[f(P(\theta^{(t+1)}))] & \leq \mathbb{E}_{ik}[f(\theta^{(t)}- \eta_t \nabla f_{ik}(\theta^{(t)}))],
\end{align*}
which shows that the rate of convergence in expectation of U-max is at least as fast as that of standard SGD.

The proof trivially generalizes to sampling multiple datapoints and classes per iteration by replacing $ \log(1+e^{x_i^\top (w_k - w_{y_i})})$ with $\log(1+\sum_{j=1}^me^{x_i^\top (w_{k_j} - w_{y_i})})$.

\end{proof}

\newpage
\section{Results over runtime}\label{app:runtime_results}

\begin{table}[h]
\caption{Time in seconds taken to run 50 epochs. OVE/NCE/IS/Vanilla/U-max with $n=1,m=5$ all have the same runtime. Implicit SGD with $n=1,m=1$ is faster per iteration. The final column displays the ration of OVE/.../U-max to Implicit SGD for each dataset.}
\label{tbl:runtimes}
\vskip 0.15in
\begin{center}
\begin{small}
\begin{sc}
\begin{tabular}{lccc}
\toprule
Data set & Implicit SGD & OVE/NCE/IS/Vanilla/U-max & Ratio \\
\midrule
MNIST & 1283 & 2494 & 1.94 \\
Bibtex & 144 & 197 & 1.37 \\
Delicious & 287 & 325 & 1.13 \\
Eurlex & 427 & 903 & 2.12 \\
AmazonCat & 24392 & 42816 & 1.76 \\
Wiki10 & 783 & 1223 & 1.56 \\
WikiSmall & 6407 & 8470 & 1.32 \\
\midrule
Average & - & - & 1.60 \\
\bottomrule
\end{tabular}
\end{sc}
\end{small}
\end{center}
\vskip -0.1in
\end{table}

\begin{figure*}[h]
\centering
\begin{minipage}{.24\textwidth}
  \centering
  \includegraphics[width=.59\linewidth]{plots/Legend}
\end{minipage}%
\hfill
\begin{minipage}{.24\textwidth}
  \centering
  \includegraphics[width=.99\linewidth]{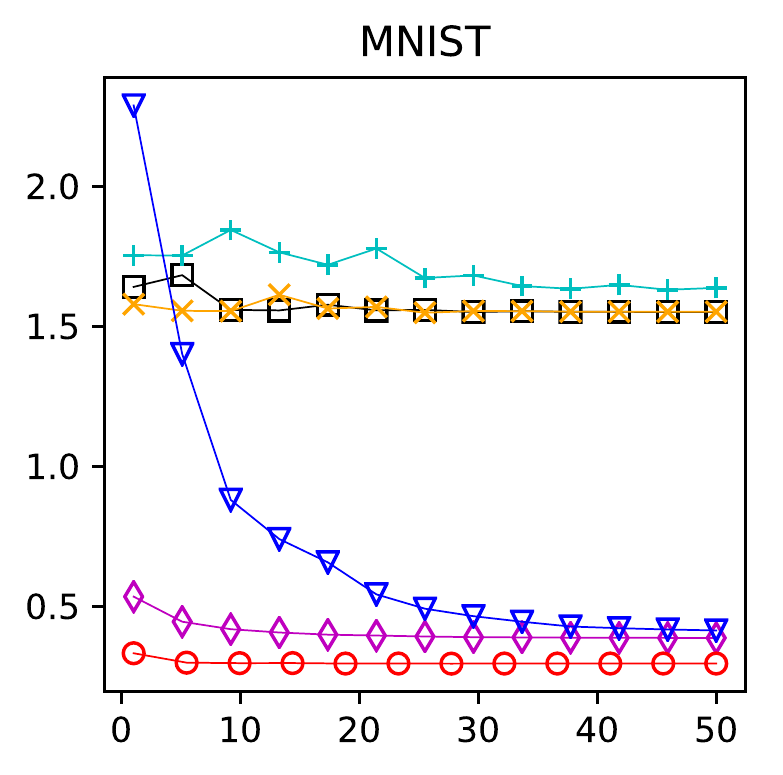}
\end{minipage}
\begin{minipage}{.24\textwidth}
  \centering
  \includegraphics[width=.99\linewidth]{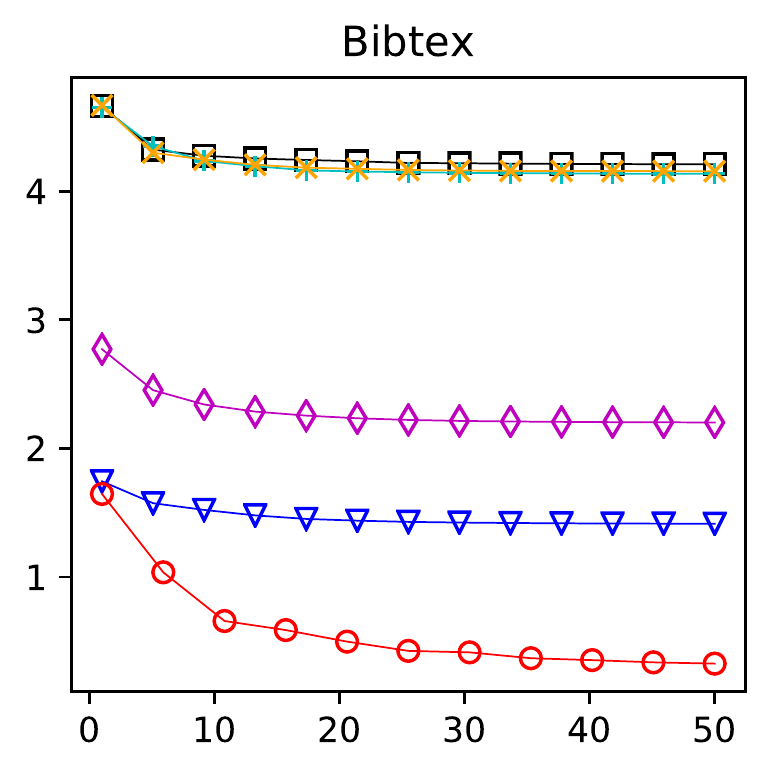}
\end{minipage}%
\hfill
\begin{minipage}{.24\textwidth}
  \centering
  \includegraphics[width=.99\linewidth]{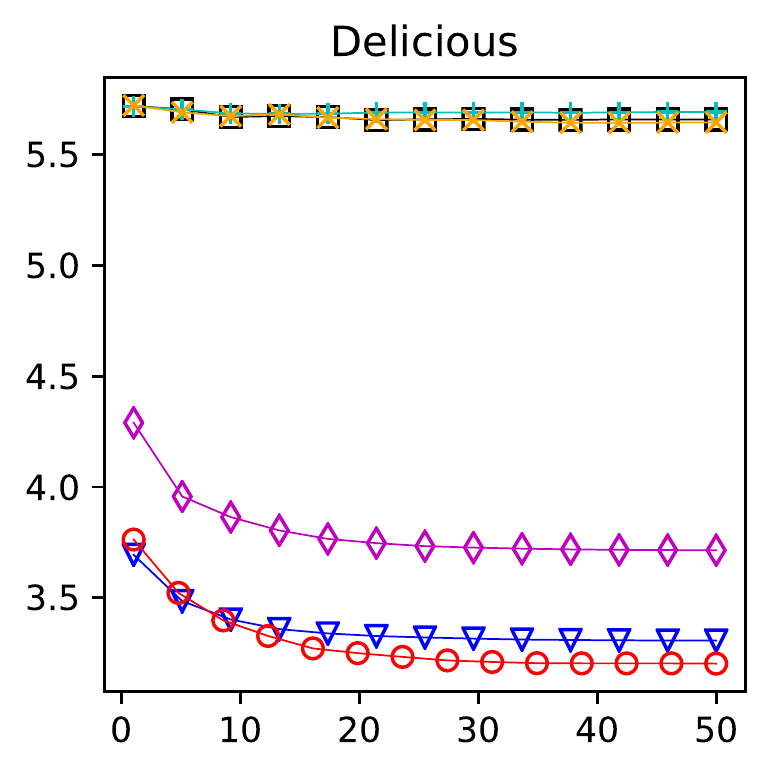}
\end{minipage}
\begin{minipage}{.24\textwidth}
  \centering
  \includegraphics[width=.99\linewidth]{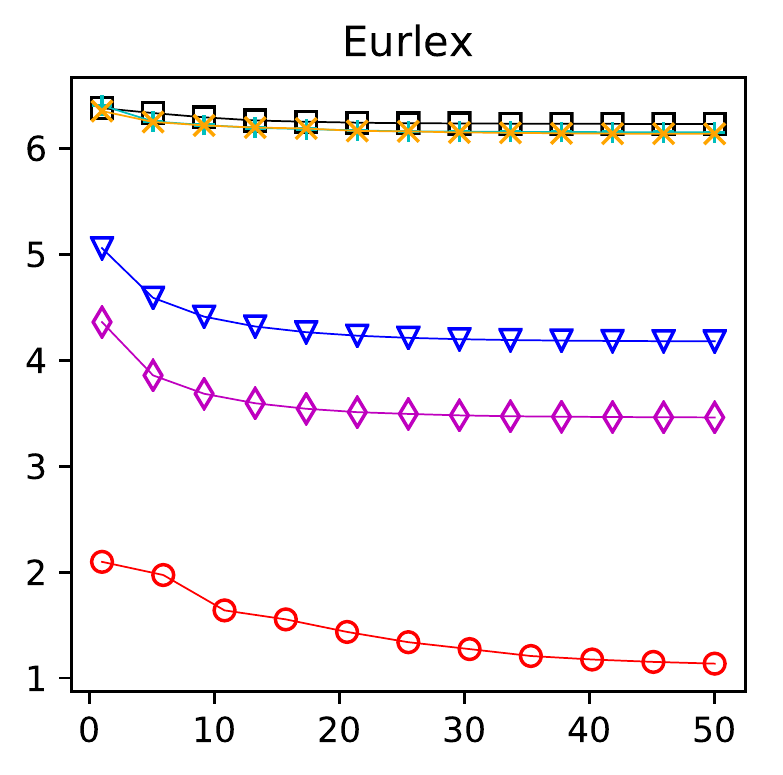}
\end{minipage}%
\hfill
\begin{minipage}{.24\textwidth}
  \centering
  \includegraphics[width=.99\linewidth]{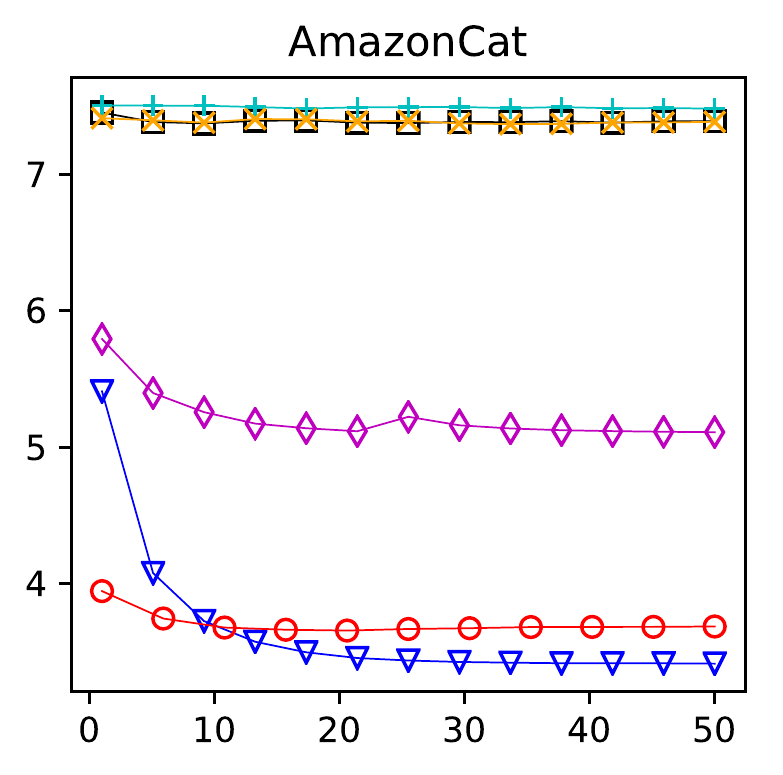}
\end{minipage}
\begin{minipage}{.24\textwidth}
  \centering
  \includegraphics[width=.99\linewidth]{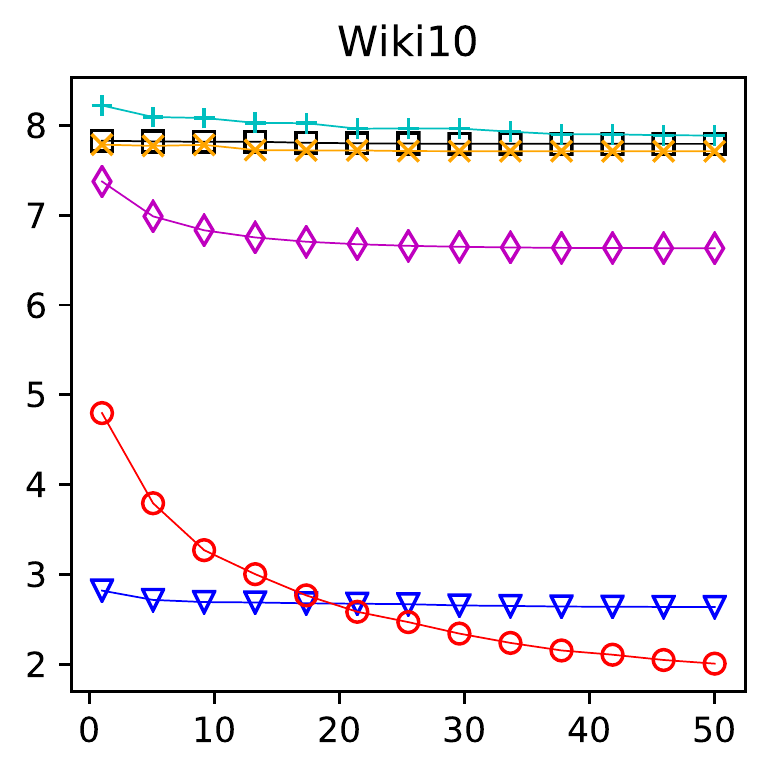}
\end{minipage}%
\hfill
\begin{minipage}{.24\textwidth}
  \centering
  \includegraphics[width=.99\linewidth]{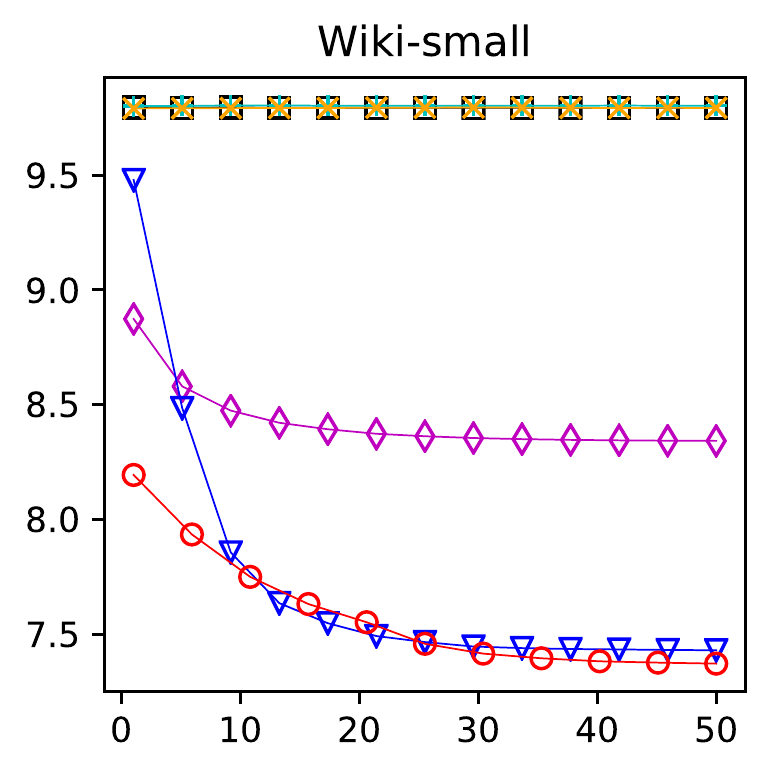}
\end{minipage}
\caption{The x-axis is runtime measured by the number of epochs for OVE, NCE, IS, vanilla SGD and U-max (they all have the same runtime). Since Implicit SGD is faster than these methods, more epochs are plotted for it. The number of Implicit SGD epochs is equal to 50 times by the ratio displayed in Table~\ref{tbl:runtimes} for each dataset. The y-axis is the log-loss from~(\ref{eq:original_log_likelihood}).}
\label{fig:comparison_prediction_log_loss_runtime}
\end{figure*}

\end{document}